\documentclass{article}

\usepackage{graphicx}
\usepackage{booktabs} 
\usepackage{hyperref}

\usepackage[accepted]{ourstyle}

\usepackage{amsmath, amssymb, amsthm}
\usepackage{mathtools}
\usepackage{hyperref}
\usepackage{todonotes}

\newtheorem{ther}{Theorem}

\newcommand{\ta}{\widetilde{\mathcal A}}

\begin{document}

\twocolumn[
\icmltitle{Universal scaling laws in the gradient descent training of neural networks}

\begin{icmlauthorlist}
\icmlauthor{Maksim Velikanov}{sk}
\icmlauthor{Dmitry Yarotsky}{sk}
\end{icmlauthorlist}

\icmlaffiliation{sk}{CDISE, Skolkovo Institute of Science and Technology, Moscow, Russia. Emails: maksim.velikanov@skoltech.ru, d.yarotsky@skoltech.ru}

\vskip 0.3in
]

\printAffiliationsAndNotice{}

\begin{abstract}
Current theoretical results on optimization trajectories of neural networks trained by gradient descent typically have the form of rigorous but potentially loose bounds on the loss values. In the present work we take a different approach and show that the learning trajectory can be characterized by an explicit asymptotic at large training times. Specifically, the leading term in the asymptotic expansion of the loss behaves as a power law $L(t) \sim t^{-\xi}$ with exponent $\xi$ expressed only through the data dimension, the smoothness of the activation function, and the class of function being approximated. Our results are based on spectral analysis of the integral operator representing the linearized evolution of a large network trained on the expected loss. Importantly, the techniques we employ do not require specific form of a data distribution, for example Gaussian, thus making our findings sufficiently universal.        
\end{abstract}

\section{Introduction}

A major challenge in the research of neural networks is the quantitative theoretical description of their optimization by gradient descent. At present, many aspects of network training seem to be understood rather well on a qualitative level, or admit convincing heuristic explanations, but we seem to lack tools for making reasonably accurate quantitative predictions, even for relatively simple models and data. In this sense, the theory of neural networks compares unfavorably to physics, which is also an application-driven field but with an apparently much more successful penetration of theoretical methods. The main difficulty here is probably the complex structure of the data and models, which are hard to describe in terms of convenient and simple mathematical abstractions. 

In recent years, a significant progress in the theoretical analysis of gradient descent of neural networks has been associated with the limit of large networks, which can be studied using various methods from partial differential equations \cite{mei2018mean, rotskoff2018trainability}, kernel methods \cite{NEURIPS2018_5a4be1fa,NEURIPS2019_0d1a9651}, spin glass theory \cite{choromanska2015loss}, random matrix theory \cite{pennington2017geometry}, dynamical systems \cite{poole2016exponential}, and other mathematical fields. 

In the present work, we consider a setting of large networks and large, smoothly distributed data sets that allows us to obtain an explicit leading term in the long-time evolution of the loss under gradient descent. We are mainly inspired by the spectral theory of singular integral operators \cite{Birman_1970}, which we apply to the linearized evolution of the network. While this linearized evolution has been widely studied recently, most related research seems to focus on theoretical convergence guarantees and upper bounds for the loss values \cite{pmlr-v97-arora19a,nitanda2021optimal}, or on a highly symmetric problems admitting explicit solution \cite{NEURIPS2019_253f7b5d,yang2019a}. In contrast, we focus on obtaining explicit loss evolution formulas, which we find in the form of power laws 
\begin{equation}\label{eq:loss_asymp1}L(t)\sim Ct^{-\xi}.\end{equation}
We argue that the exponents $\xi$ here exhibit some form of universality, in that they are essentially determined by the input dimension $d$ and by the smoothness classes of the activation function and the target function. In particular, we find that in the case of ReLU networks approximating 
an indicator function of some region in the $d$-dimensional space (a classification problem target), the natural value of the exponent is $\xi=\tfrac{1}{d+1}$. On the other hand, in the case of target functions generated by a 
randomly initialized wide ReLU network, the exponent is $\xi=\tfrac{3}{d+1}.$ Our approach also allows us to obtain explicit expressions for the coefficient $C$ in these cases.

\begin{figure*}
    \centering
    \includegraphics[width=1.0\textwidth,clip,trim= 0 5mm 0 0]{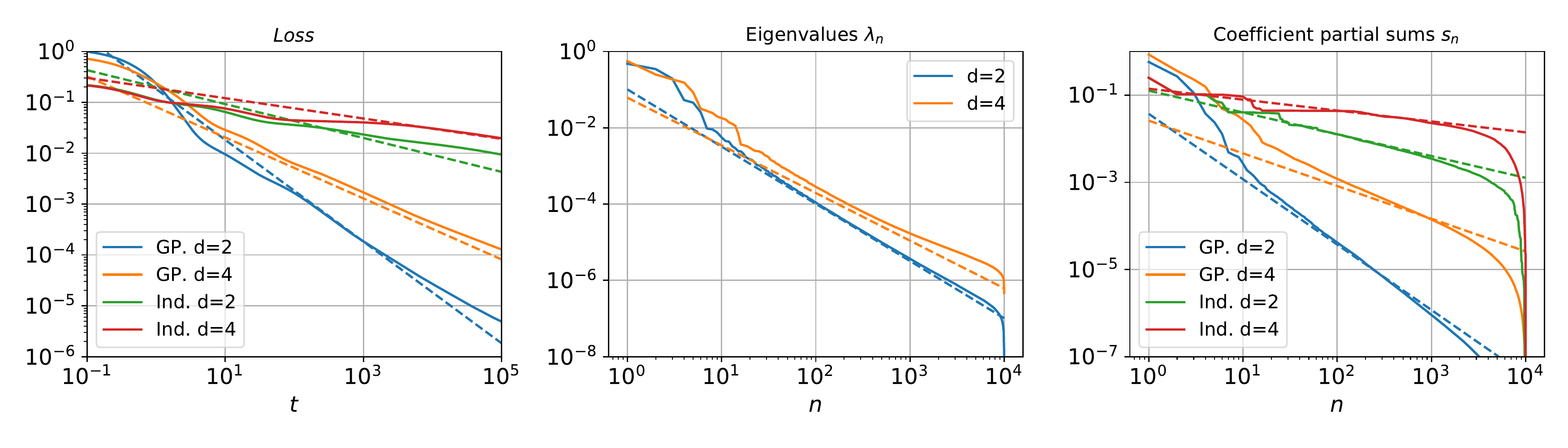}
    \caption{The figure shows the loss trajectory and spectral properties of the neural tangent kernel of a shallow network in the NTK regime.  The target function (corresponding to the initial displacement between the network output and the approximated function) is either generated by a Gaussian process (\emph{GP}) modeled by a larger network of the same architecture, or is an indicator function of a $d$-dimensional ball (\emph{Ind}). The data distributions $\mu$ are modeled as mixtures of 8 Gaussian distributions with random centers, and the data dimension is either $d=2$ or $d=4$. The \textbf{solid} lines show the numerically obtained values, while the \textbf{dashed} lines show the respective theoretical power-law asymptotics. The dataset size is $M=10^4$ (see Section \ref{sec:exp_det} (SM) for further details of experiments). \newline
    \textbf{Left:} Loss evolution for a shallow network with width $N=3000$. The scaling exponent giving the slope of the theoretical asymptotic is $\xi=\tfrac{\beta}{d+\alpha}=\tfrac{3}{d+1}$ for GP and $\xi=\tfrac{1}{d+\alpha}$ for Ind  (see Section \ref{sec:coeff}). \textbf{Center:} Distribution of the infinite network NTK eigenvalues $\lambda_n$. The theoretical scaling exponent is $\nu=1+\tfrac{1}{d}$ (see Section \ref{sec:eigvalues_asym_base}). \textbf{Right:} Distributions of the coefficient partial sums $s_n$ (see Eq.~\eqref{eq:s_partial_sum}). The theoretical scaling exponent is $\kappa=\tfrac{\beta}{d}=\tfrac{3}{d}$ for GP and $\kappa=\tfrac{1}{d}$ for Ind (see Section \ref{sec:coeff}).}
    \label{fig:loss_dyn}
\end{figure*}  

The power law \eqref{eq:loss_asymp1} is established using similar power laws (but with different exponents and coefficients) for the eigenvalues of the evolution operator and for the coefficients in the expansion of the target function over corresponding eigenvectors. 
These power laws are indeed confirmed by our experiments (see Fig.~\ref{fig:loss_dyn}).   

Our main scenario is approximation by shallow ReLU network in the NTK regime, but we also briefly consider several modifications of this scenario, namely the activation functions $(x_+)^q$ with $q>0$, approximation by a deep network in the NTK regime, and approximation in the mean field regime.

\section{Related work}

The approximation of linearized network evolution and its applications  were studied in many works, see in particular \cite{NEURIPS2019_0d1a9651,chizat2019lazy,NEURIPS2020_b7ae8fec,lewkowycz2020large,NEURIPS2020_c6dfc6b7}. The role of the eigenvalues and eigenvectors of the NTK in the linearized network was emphasized in \cite{pmlr-v97-arora19a}, where the GD dynamic of the finite network was linked to the dynamic of its infinite width counterpart, determined by spectral properties of the corresponding NTK. In subsequent works, the NTK spectrum was central for description of network training \cite{nitanda2021optimal} and generalization \cite{canatar2021spectral, bahri2021explaining}. These papers use the assumption of power law NTK spectrum, but justify it empirically or for highly symmetric problems. 

Because of the importance of the NTK spectrum, a number of works focused on its description in different settings. We first mention directions and settings that are different from ours. The case of very deep networks was studied in \cite{pmlr-v119-xiao20b,hayou2021meanfield, pmlr-v97-hayou19a, pmlr-v132-agarwal21b}. This analysis is relying on convergence of the NTK to simple fixed points in the limit of infinite depth. \cite{pmlr-v97-hayou19a} also studied the choice of activation function, in particular its smoothness. Another line of research \cite{pmlr-v119-adlam20a, NEURIPS2020_572201a4} uses techniques from Random Matrix Theory to analyze the setting where the dataset size $M$ goes to infinity together with data dimension $d$ and layer widths $n_l$.

In this work we consider the setting of fixed data dimension and effectively infinite network width and dataset size. In this case \cite{NEURIPS2019_253f7b5d} showed that the network evolution can be described by a deterministic integral operator that is easier to analyze than a large but finite matrix. Also, in that paper and in the papers \cite{yang2019a, cao2020understanding, nitanda2021optimal}, the integral operator was explicitly diagonalized in the special case of uniform distribution on a sphere.

A power law upper bound for the training loss was obtained in  \cite{nitanda2021optimal}, but with an exponent $\xi$ smaller than ours. The paper \cite{bahri2021explaining} describes a different power law, relating the test loss at the end of training to the dataset size and the network width. This result also relies on a power law assumption for the NTK spectrum. 

Another related line of research is the case of univariate functions. The gradient descent evolution of univariate shallow networks has been analytically studied in the papers  \cite{yarotsky2018collective, williams2019gradient}.

\section{Asymptotic evolution of the loss function}
\label{sec:loss_asym}

We consider a linearized training of a neural network by gradient descent. Such linear approximations arise naturally in various ``lazy training'' scenarios \cite{chizat2019lazy}. Consider the standard quadratic loss function
\begin{equation}(\mathbf W)=\tfrac{1}{2}\int_{\mathbb R^d}|\widetilde f(\mathbf W,\mathbf x)-f(\mathbf x)|^2 \mu(\mathbf x)d\mathbf x,\end{equation}
where $\mathbf x$ is the $d$-dimensional input, $f$ is the approximated function, $\widetilde f$ is the network, $\mathbf W$ are the network weights, and $\mu$ is the data distribution on which the network is trained. Gradient descent can be written as the differential equation 
\begin{equation}\tfrac{d}{dt}\mathbf W=-\nabla_{\mathbf W} L(\mathbf W).\end{equation}

We assume that the weight vector $\mathbf W$ is close to a global minimum $\mathbf W_*$ where $L(\mathbf W_*)=0$ and $\widetilde f(\mathbf W_*,\mathbf x)=f(\mathbf x)$ $\mu$-a.e., and that the evolution equation can be linearized at $\mathbf W=\mathbf W_*$. It is convenient to write this linearized equation in terms of the difference $\delta f(\mathbf x)=\widetilde f(\mathbf W,\mathbf x)- f(\mathbf x)$ between the current output and the target. 
The corresponding linear equation is
\begin{equation}\label{eq:dfdta}\tfrac{d}{dt}\delta f=- \mathcal A\delta f,\end{equation}
where $A$ can be written as the integral operator 
$$\mathcal A\delta f(\mathbf x)=\int_{\mathbb R^d}\Theta(\mathbf x,\mathbf x')\mu(\mathbf x')\delta f(\mathbf x')d\mathbf x'$$
with the NTK (neural tangent kernel)
\begin{equation}\label{NTK_def}
    \Theta(\mathbf x,\mathbf x')=\nabla_{\mathbf W}\widetilde f(\mathbf W_*,\mathbf x)^T\nabla_{\mathbf W}\widetilde f(\mathbf W_*,\mathbf x')
\end{equation}
The evolution operator $\mathcal A$ is a symmetric non-negative definite operator with respect to the scalar product $\langle f_1, f_2\rangle_{\mu} = \int f_1(\mathbf x) f_2(\mathbf x) \mu(\mathbf x) d\mathbf x$. By multiplying functions $f$ by $\mu^{1/2}$, the operator $\mathcal A$ can be brought to the form \begin{equation}\label{eq:wta}\widetilde{\mathcal A}=\mu^{1/2}{\mathcal A}\mu^{-1/2}\end{equation} with a symmetric kernel,
\begin{equation}\label{eq:wtadf0}\widetilde{\mathcal A}\delta f(\mathbf x)=\int_{\mathbb R^d}\mu^{1/2}(\mathbf x)\Theta(\mathbf x,\mathbf x')\mu^{1/2}(\mathbf x')\delta f(\mathbf x')d\mathbf x'.\end{equation}
In this form $\widetilde{\mathcal A}$ is symmetric w.r.t. the usual scalar product $\langle f_1, f_2\rangle = \int f_1(\mathbf x) f_2(\mathbf x)  d\mathbf x$. 
Observe that the loss at time $t$ can be written as 
\begin{equation}
    L(t)=\tfrac{1}{2}\|g_t\|^2,  \quad g_t=e^{-t\ta}g,  
\end{equation}
where the norm $\|\cdot\|$ corresponds to the scalar product $\langle\cdot,\cdot\rangle$ and the function $g$ is given by
\begin{equation}\label{eq:g}
g(\mathbf x)=\mu^{1/2}(\mathbf x)(\widetilde f(\mathbf w(t=0),\mathbf x)-f(\mathbf x)).
\end{equation}

We can now describe the evolution of the loss by diagonalizing the operator $\ta$. We will be interested in the scenario where $\mu$ is a smooth function that is compactly supported or falls off at infinity sufficiently fast. (In particular, in the context of a finite training set this means that this set is  large enough to be legitimately approximated by $\mu$.) In this case, for typical kernels $\Theta,$ the operator $\ta$ will have a discrete spectrum with eigenvalues converging to 0. Let $\lambda_n$ denote the eigenvalues of $\ta$ sorted in decreasing order, and let $c_n$ be the respective coefficients in the expansion of the initial error $g$ (given by Eq.~\eqref{eq:g}) over the normalized orthogonal eigenvectors. Then, the evolution of the loss takes the form
\begin{equation}
\label{loss}
    L(t) = \frac{1}{2} \sum\limits_{n=0}^\infty e^{-2 \lambda_n t} |c_n|^2.
\end{equation}
To compute the asymptotic of $L(t)$ at large times $t$, we need to know the distribution of the eigenvalues $\lambda_n$ and the coefficients $c_n$ at large $n$. The key assumption of our work (verified later for certain scenarios) is that these distributions have power law forms. Specifically, regarding the eigenvalues $\lambda_n$ we assume that 
\begin{align}
    \label{lambda asym}
    &\lambda_n \sim \Lambda n^{-\nu} 
\end{align}
with some coefficient $\Lambda$ and exponent $\nu$. 
Regarding the coefficients $c_n$, we assume that they also have a power law distribution on a large length scale in $n$, but possibly deviate from this law locally (e.g., due to oscillations). For this reason, it is convenient to describe their large $n$ behavior by the partial sums
\begin{equation}\label{eq:s_partial_sum}
    s_n \equiv \sum\limits_{k \geq n} |c_k|^2.
\end{equation}
We then assume that 
\begin{align}
    \label{tail asym}
    &s_n \sim K n^{-\kappa}
\end{align}
with some coefficient $K$ and exponent $\kappa.$

Under assumption of the  power laws \eqref{lambda asym} and \eqref{tail asym}, it is easy to check (see SM, Section \ref{sec:genlossformula}) that the loss also has a power law asymptotic \eqref{eq:loss_asymp1} with the constant $C$ and exponent $\xi$ expressible through the constants $\Lambda,K$ and exponents $\nu,\kappa:$
\begin{equation}\label{eq:loss_asymp}
    L(t) \sim \frac{K}{2} \Gamma\left(\frac{\kappa}{\nu} + 1\right) \left( 2 \Lambda  t\right)^{-\frac{\kappa}{\nu}},
\end{equation}
where $\Gamma(z)$ is the Gamma function.

In Fig.~\ref{fig:loss_dyn} we illustrate this approach to the analysis of long-term loss evolution with several examples of target functions having different smoothness and dimension and, as a result, exhibiting different exponents.

In the remainder of the paper we show that the power laws \eqref{lambda asym} and \eqref{tail asym}, and hence the large-$t$ asymptotic \eqref{eq:loss_asymp} of the loss, are indeed valid for some natural network training scenarios. The asymptotic \eqref{lambda asym} of the eigenvalues is primarily determined by the singularities of the kernel $\Theta$. These singularities can be explicitly described for shallow neural network with piecewise smooth activations such as ReLU. Then, the power law \eqref{lambda asym} can be derived from general results on integral operators with singular kernels. In particular, in the case of ReLU we find that $\nu=\tfrac{1}{d}.$ 

The asymptotic \eqref{tail asym} of the coefficients is more subtle, as it depends significantly on the class of the initial error function $g$ (which in turn depends, by Eq.~\eqref{eq:g}, on the target function $f$ and the initial approximation $\widetilde f$). We derive this power law for one natural class of discontinuous functions $g$, and for functions $g$ that are realizations of a Gaussian process of a particular ``roughness''.  

Moreover, for these two classes, we also find an explicit form of the coefficient $C$ appearing in loss asymptotic \eqref{eq:loss_asymp1}. This requires us, however, to modify the above derivation of the loss asymptotic \eqref{eq:loss_asymp} 
by what can be called ``integrated localization''. Roughly speaking, in the large-$t$ limit we can think of the eigenvectors of $\ta$ as infinitesimally localized in $\mathbb R^d$. We then apply the above derivation of Eq.~\eqref{eq:loss_asymp} not to the full set of eigenvectors, but separately to each infinitesimal sub-domain, and then integrate the results. See details in Sections \ref{sec:coeff} and \ref{sec:SM_coeffs} (SM).

The subsequent exposition is structured as follows. In Section \ref{sec:background} we provide the background on the NTK kernel and spectral properties of singular integral operators. Then, in Section \ref{sec:asymp}, we derive the loss asymptotic \eqref{eq:loss_asymp1} for our main setting -- the NTK training with a shallow ReLU network. After that, in Section \ref{sec:extensions} we consider various modifications of this setting: other activation functions (Section \ref{sec:diffsmooth}), deep networks (Section \ref{sec:deep}), and training in the mean field regime (Section \ref{sec:mf}). 

\section{Background}\label{sec:background}

\subsection{Infinitely wide networks}
Lazy training scenarios discussed in section \ref{sec:loss_asym} naturally arise for the networks in the limit of infinite width. However, there are several ways to scale parameters of the network with width, which lead to different operating regimes of infinitely wide networks \cite{pmlr-v119-golikov20a, golikov2020dynamically}. 

\subsubsection{NTK regime}
The first option we consider is the renowned NTK regime \cite{NEURIPS2018_5a4be1fa}, for which the NTK $\Theta$ defined in \eqref{NTK_def} is deterministic and constant during training. This simplification immediately leads to a linear dynamic in the space of network outputs. In this work we focus on feed-forward networks parametrized as
\begin{equation}
    \begin{cases}
    &z^1_j =  \sum\limits_{i=1}^{d} \sigma_w w^1_{ij} x_i+ \sigma_b b^1_j \\
    &z^l_j = \sum\limits_{i=1}^{n_{l-1}} \frac{\sigma_w}{\sqrt{n_{l-1}}} w^l_{ij} x^{l-1}_i+ \sigma_b b^l_j, \quad l>1 \\
    &x^l_j = \phi(z^l_j)
    \end{cases}
\end{equation}
Here $n^l$ is the width of layer $l$, $x_i$ is a network input and $z^L_j$ is the network output. We also consider the last layer without bias term $b^L_j$ and having width $n_L=1$ (scalar output). Trainable parameters $w^l_{ij}, \; b^l_j$ are initialized as i.i.d. normal Gaussians. 

The output of each layer $l$ at initialisation is a Gaussian process with covariance $\langle z^l_j(\mathbf{x})z^l_{j'}(\mathbf{x}')\rangle=\delta_{jj'}\Sigma^{(l)}(\mathbf{x},\mathbf{x}')$. By introducing the NTK's $\Theta^{(l)}(\mathbf{x},\mathbf{x'})$ of intermediate layers,  one can recursively compute \cite{NEURIPS2019_0d1a9651} the both NTK and covariance 
\begin{equation}\label{eq:recur0}
    \begin{cases}
    & \Sigma^{(l+1)} = \sigma_w^2 \big\langle \phi(z^l)\phi(z'^{l}\big\rangle + \sigma_b^2\\
    & 
    \Theta^{(l+1)} = \Sigma^{(l+1)} +\sigma_w^2 \Theta^{(l)}\big\langle \dot{\phi}(z^l)\dot{\phi}(z'^l)\big\rangle
    \end{cases}
\end{equation}
Here $z^l = z^l(\mathbf{x})$ and $z'^l=z^l(\mathbf{x'})$ are draws from the Gaussian process with covariance $\Sigma^{(l)}(\mathbf{x},\mathbf{x}')$, and $\langle\ldots\rangle$ is the averaging.  

To analyze the spectrum of the evolution operator $\ta$ given by \eqref{eq:wtadf0} we will intensively use the explicit form of the NTK and covariance. For the ReLU activation $\phi(z)=(z)_+$ the averages in \eqref{eq:recur0} can be computed analytically \cite{NIPS2009_kernel}. When considering only a pair of points $\mathbf{x},\mathbf{x}'$, it is convenient to parametrize the covariance at these points in some layer $l$ as 
\begin{equation}\label{eq:covariance_param0}
    \Sigma^{(l)}(\mathbf{x},\mathbf{x}') = 
    \begin{pmatrix}
    &r_l^2 & r_l r_l' \cos \varphi_l\\
    &r_l r_l' \cos \varphi_l & r_l'^2
    \end{pmatrix}
\end{equation}
In the case of shallow network ($L=2$) one gets the following explicit expressions for the  NTK and covariance of the network output:
\begin{align}
    & \label{ReLU covariance}
    \Sigma(\mathbf{x},\mathbf{x'}) = \frac{\sigma_w^2}{2\pi} r r' \big(\sin\varphi +\cos\varphi(\pi-\varphi)\big) \\
    & \label{ReLU ntk}
    \Theta(\mathbf{x},\mathbf{x'}) = \Sigma(\mathbf{x},\mathbf{x'}) + \frac{\sigma_w^2}{2\pi} r r' \cos\varphi (\pi-\varphi)
\end{align}
Here $r,r',\varphi$ are parameters from \eqref{eq:covariance_param0} with $l=1$ (we dropped index $1$). They have a clear interpretation in terms of extended input vectors $\tilde{\mathbf{x}}=(\sigma_w \mathbf{x},\sigma_b) \in \mathbb{R}^{d+1}$. Specifically, $r=\|\tilde{\mathbf{x}}\|, \: r'=\|\tilde{\mathbf{x'}}\|$ and $\varphi$ is the angle between $\tilde{\mathbf{x}}$ and $\tilde{\mathbf{x'}}$. 

\subsubsection{Mean Field regime}
This operating regime of infinitely wide networks is naturally defined \cite{mei2018mean,rotskoff2018trainability} for the shallow networks of the form
\begin{equation}
    f(\mathbf{W},\mathbf{x}) = \frac{1}{N} \sum\limits_{i=1}^N c_i \phi(\mathbf{w}_i\cdot \mathbf{x} + b_i) = \frac{1}{N} \sum\limits_{i=1}^N \tilde{\phi}(\tilde{\mathbf{w}}_i, \mathbf{x})
\end{equation}

Here $\tilde{\mathbf{w}}_i=(c_i,\mathbf{w}_i,b_i)$ denotes the collection of parameters, assiciated with a single neuron. In the infinite width limit $N \rightarrow \infty$ the evolution is described by a PDE on parameter density distribution $p(\tilde{\mathbf{w}})$:
\begin{align}
    &
    \partial_t p =  \nabla_{\tilde{\mathbf{w}}}\Big[p\nabla_{\tilde{\mathbf{w}}}\int K(\tilde{\mathbf{w}},\tilde{\mathbf{w}}') (p(\tilde{\mathbf{w}}') - p_\infty(\tilde{\mathbf{w}'}))d\tilde{\mathbf{w}}'\Big] \nonumber\\
    & \label{eq:MF_eq} 
    K(\tilde{\mathbf{w}},\tilde{\mathbf{w}}') = \int \tilde{\phi}(\tilde{\mathbf{w}}, \mathbf{x})\mu(\mathbf{x})\tilde{\phi}(\tilde{\mathbf{w}}', \mathbf{x})d\mathbf{x}
\end{align}
Under mild assumptions \cite{NEURIPS2018_a1afc58c}, solution of the MF equation converges to the global optimum $p_\infty(\tilde{\mathbf{w}}')$. At large times $t$ the MF equation can be linearized \cite{yarotsky2018collective} around $p_\infty$, thus bringing the network to the lazy training regime discussed in section \ref{sec:loss_asym}. The network NTK in this case is equal to
\begin{equation}
    \label{eq:MF_NTK0}
    \Theta(\mathbf{x},\mathbf{x'}) = \int  \nabla_{\tilde{\mathbf{w}}}\tilde{\phi}(\tilde{\mathbf{w}}, \mathbf{x}) p_\infty(\tilde{\mathbf{w}}) \nabla_{\tilde{\mathbf{w}}}\tilde{\phi}(\tilde{\mathbf{w}}, \mathbf{x}')d\tilde{\mathbf{w}}
\end{equation}
\subsection{Singular integral operators}\label{sec:singular}

Consider the evolution operator $\widetilde{\mathcal A}$ given by Eq.~\eqref{eq:wtadf0}. Under our assumptions, the kernel $\mu^{1/2}(\mathbf x)\Theta(\mathbf x,\mathbf x')\mu^{1/2}(\mathbf x')$ of this operator quickly falls off at infinity and is  smooth outside the diagonal $\mathbf x=\mathbf x'$, but, as we will see later, has a  homogeneous singularity on this diagonal.  In this setting, a general theory developed in \cite{Birman_1970} allows us to obtain the asymptotic distribution \eqref{lambda asym} of the eigenvalues with explicit constants $\Lambda$ and $\nu$.

Specifically, suppose that in a neighborhood of the diagonal the kernel $\Theta$ has a representation 
\begin{equation}\label{eq:theta10}\Theta(\mathbf x,\mathbf x')=\theta_{\mathbf x}(\mathbf{x}-\mathbf{x}')+\ldots,\end{equation}
where $\theta_{\mathbf x}(\cdot)$ is a (possibly $\mathbf x$-dependent) even ($\theta_{\mathbf x}(\mathbf z)=\theta_{\mathbf x}(-\mathbf z)$) homogeneous function of degree $\alpha$:
\begin{equation}\label{eq:theta20}\theta_{\mathbf x}(c\mathbf z)=|c|^\alpha\theta_{\mathbf x}(\mathbf z),\end{equation}
and the dots $\ldots$ denote terms of higher smoothness.

Let $N_\lambda$ denote the number of eigenvalues of $\ta$ greater than $\lambda$. Then, it is shown in \cite{Birman_1970} that for small $\lambda$, the leading term of $N_\lambda$ is given by
\begin{equation}\label{eq:nlambda}
N_\lambda \sim \Big(\int \gamma_\mathbf x\mu^{\tfrac{d}{d+\alpha}}(\mathbf x)d\mathbf x\Big)\lambda ^{-\tfrac{d}{d+\alpha}}.\end{equation}
Here, $\gamma_{\mathbf x}$ is defined as follows. Let $\widetilde \theta_{\mathbf x}$ denote the Fourier transform of the homogeneous function $\theta_{\mathbf x}$ defined using the Riesz summation formula:
\begin{equation}\label{eq:sing_Fourier}
    \widetilde \theta_{\mathbf x}(\mathbf k)=\lim_{r\to\infty}\int_{\|\mathbf x\|<r}(1-\tfrac{|\mathbf x|^2}{r^2})^c \theta(\mathbf x)e^{-i\mathbf k\cdot\mathbf x}d\mathbf x.
\end{equation}
For sufficiently large $c$ the limit exists and is independent of $c$, and the resulting function $\widetilde \theta_{\mathbf x}$ is a homogeneous function of degree $-(d+\alpha)$. We then set  
\begin{equation}\label{eq:sing_Fourier_vol}
\gamma_{\mathbf x}=(2\pi)^{-d}|\{\mathbf k\in\mathbb R^d: \widetilde \theta_{\mathbf x}(\mathbf k)>1\}|,
\end{equation}
where $|\cdot|$ denotes the Lebesgue measure.

Formula \eqref{eq:nlambda} can be derived as follows. Divide the domain $\mathbb R^d$ into multiple small subsets $\Omega_m$, and think of $\widetilde{\mathcal A}$ as an operator matrix corresponding to the decomposition $L^2(\mathbb R^d)=\oplus_m L^2(\Omega_m)$. Using the fall off and smoothness of the kernel outside the diagonal, one can show that the leading term of $N_\lambda$ is determined only by the diagonal elements of this operator matrix. Then, the leading term can be found by considering  each restriction $\widetilde{\mathcal A}|_{L^2(\Omega_m)}$ separately and summing the respective contributions to $N_\lambda$:
$$N_\lambda (\widetilde{\mathcal A})\sim \sum_{m} N_\lambda (\widetilde{\mathcal A}|_{L^2(\Omega_m)}).$$
If we decrease the size of each $\Omega_m$ by a factor $M$, then the number of terms in this sum increases $M^d$-fold, but at the same time each $N_\lambda (\widetilde{\mathcal A}|_{L^2(\Omega_m)})$ decreases also roughly $M^d$-fold, due to rescaling of eigenvalues. In the limit of infinitely small subsets $\Omega_m,$ the operator $\widetilde{\mathcal A}|_{L^2(\Omega_m)}$ can be approximated by the convolution with the homogeneous function $\theta_{\mathbf x}.$ The asymptotic form of its eigenvalues can then be written in terms of the Fourier transform $\widetilde \theta_{\mathbf x}$ as given above. The power $\tfrac{d}{d+\alpha}$ in Eq.~\eqref{eq:nlambda} can be deduced by observing that the volume of the $\mathbf k$-space corresponding to $\widetilde \theta_{\mathbf x}>\lambda$ scales as $\lambda^{-\tfrac{d}{d+\alpha}}.$ 

The formula \eqref{eq:nlambda} can be translated into the power law \eqref{lambda asym} by inverting the relation between $\lambda$ and $n$ (note that $N_{\lambda_n}=n$ for any $n$). Specifically, we find that the law  \eqref{lambda asym} holds with
\begin{equation}\label{eq:nu1ad}
    \nu = 1 + \frac{\alpha}{d},\quad
    \Lambda = \Big(\int \gamma_\mathbf x\mu^{1/\nu}(\mathbf x)d\mathbf x\Big)^{\nu}.
\end{equation}
In Section \ref{sec:coeff} we will show that this approach can be extended to yield the loss asymptotic \eqref{eq:loss_asymp1}.

\section{Asymptotic analysis of wide networks}\label{sec:asymp}
\subsection{NTK operators and their singularities}\label{sec:eigvalues_asym_base}
In this section we demonstrate the asymptotic law \eqref{lambda asym} and find constants $\nu, \Lambda$ for the shallow ReLU network in the NTK regime and with the data distribution $\mu(\mathbf{x})$ as described in Section \ref{sec:loss_asym}. The NTK of such network is given by \eqref{ReLU ntk}, \eqref{ReLU covariance} with $r,r'$ and $\varphi$ explicitly depending on input points $\mathbf{x},\mathbf{x'}$:
\begin{align}
        &\label{eq:r_x}
        r(\mathbf{x}) = \sqrt{\sigma_w^2 |\mathbf{x}|^2 + \sigma_b^2}, \quad r'\equiv r(\mathbf{x'})\\
        &\label{eq:phi_x}
        \varphi(\mathbf{x},\mathbf x') = \arccos \left(\frac{\sigma_w^2 \mathbf{x}\cdot\mathbf{x}'+\sigma_b^2}{r(\mathbf{x}) r(\mathbf{x'})}\right)
\end{align}
To use the spectral theory described in section \ref{sec:singular} we analyze the smoothness of the NTK $\Theta(\mathbf{x}, \mathbf{x'})$ defined in \eqref{ReLU ntk}, \eqref{ReLU covariance}. Firstly, $\Theta(\mathbf{x}, \mathbf{x'})$ is a smooth (infinitely differentiable) function of $r,r',\varphi$ on the whole domain. Now suppose that the bias term is not absent: $\sigma_b > 0$. Then $r(\mathbf{x})$ is a smooth function for all $\mathbf{x} \in \mathbb{R}^d$. The argument of $\arccos$ in \eqref{eq:phi_x} is also smooth everywhere, but $\arccos(z)$ itself is smooth on $(-1,1)$ and has divergent derivative at the end points $z=1,-1$ corresponding to $\varphi=0,\pi$. We see that condition $\sigma_b>0$ implies that the case $\varphi=\pi$ is never realized, while $\varphi(\mathbf{x},\mathbf{x}')=0$ at all coinciding inputs. Thus we established that $\Theta(\mathbf{x},\mathbf{x'})$ is smooth everywhere except the diagonal $\mathbf{x}=\mathbf{x'}$, where it might have a singularity.

To analyze the behavior at the diagonal, we first expand \eqref{eq:phi_x} for small $\delta\mathbf{x} = \mathbf{x}-\mathbf{x'}$ and get
\begin{equation}
\label{local angle}
    \varphi(\mathbf{x},\mathbf{x}') = \frac{\sigma_w\sqrt{r^2(\mathbf{x}) -\sigma_w^2|\mathbf{x}|^2\cos^2\psi}}{r^2(\mathbf{x})} |\delta\mathbf{x}| + O(|\delta \mathbf{x}|^2)
\end{equation}
Here $\psi$ is the angle between $\delta\mathbf{x}$ and $\mathbf{x}$. If $\sigma_b>0$, then the expression under the square root is always positive, therefore the angle $\varphi(\mathbf{x},\mathbf{x}')$ has a homogeneous singularity of degree 1 on the diagonal. Now we expand expressions \eqref{ReLU covariance} and $\eqref{ReLU ntk}$ for small $\varphi$ and find that both NTK and covariance indeed have a singularity on the diagonal with leading singular terms   
\begin{align}\label{eq:Sigma_sing}
    \Sigma_\text{sing}(\mathbf{x},\mathbf{x}') ={}& \frac{\sigma_w^2}{2\pi} r^2(\mathbf{x}) \frac{1}{3}\varphi^3(\mathbf{x},\mathbf{x}') \propto |\delta \mathbf{x}|^3\\
\label{eq:NTK_sing}
    \Theta_\text{sing}(\mathbf{x},\mathbf{x}') ={}& -\frac{\sigma_w^2}{2\pi} r^2(\mathbf{x}) \varphi(\mathbf{x},\mathbf{x}')\propto |\delta \mathbf{x}|
\end{align}
Quite importantly, only odd powers of $\varphi$ are singular, while even powers $\varphi^{2n}$ are smooth at the diagonal $\mathbf{x}=\mathbf{x'}$. This implies that the leading singular term is determined not by the smallest power $\varphi$, but by the smallest odd power.  

Now we see that both NTK $\Theta$ and covariance $\Sigma$ have the form needed for application of spectral theory discussed in section \ref{sec:singular}. For NTK we have singularity degree $\alpha=1$, which by Eq.~\eqref{eq:nu1ad} leads to the already announced exponent $\nu=1+\tfrac{1}{d}$ (see Fig.~\ref{fig:loss_dyn}). Interestingly enough, the singularity degree for the covariance is higher (namely, 3), resulting in a faster fall off of the corresponding eigenvalues. In the sequel (see Section \ref{sec:coeff}), this latter degree will appear in the analysis of loss asymptotic for target functions generated by the neural network Gaussian process, and will be denoted by $\beta$.  

\begin{figure}[t]
    \centering
    \includegraphics[width=0.48\textwidth,clip,trim= 0 5mm 0 5mm]{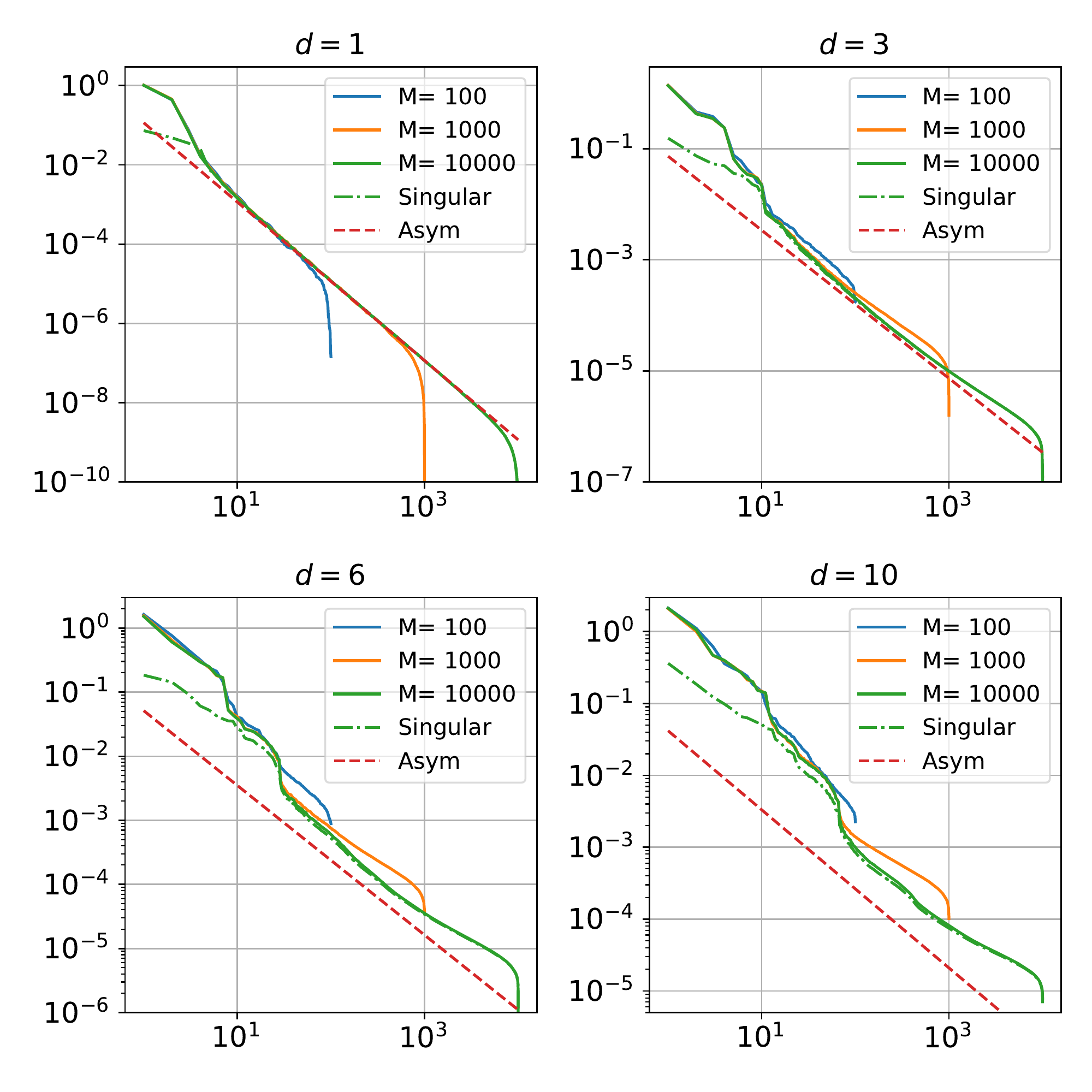}
    \caption{NTK eigenvalues $\lambda_n$ for networks with different input dimension $d$ and for different data set sizes $M$. \emph{Asym} shows the theoretical power law \eqref{lambda asym} with parameters $\Lambda,\nu$ derived in Section \ref{sec:eigvalues_asym_base}; in particular, $\nu=1+\tfrac{1}{d}$. \emph{Singular} corresponds to eigenvalues of the singular part of NTK  \eqref{eq:NTK_sing}. We see that for moderately big $n$ the singular part has the same eigenvalues as the full NTK. Observe that the number $n$ at which the spectrum converges to its asymptotic form increases with dimension $d$.} 
  \label{fig:ntk_eig_vals}
\end{figure}

In the case of NTK $\Theta(\mathbf{x,\mathbf{x'}})$ the singularity is essentially given by \eqref{local angle}. The corresponding Fourier transform in \eqref{eq:sing_Fourier} and the volume in \eqref{eq:sing_Fourier_vol} can be calculated analytically (see Section \ref{sec:gammax}). This leads to an explicit expression for the constant $\Lambda$: 
\begin{equation}\label{eq:Lambda}
\begin{split}
    \Lambda ={}& \tfrac{\sigma_w^3 \sigma_b^{\frac{1}{d}}}{(2\pi)^2} \Gamma(\tfrac{d+1}{2})  \Gamma(\tfrac{d}{2}+1)^{-(1+\frac{1}{d})}  \\
    &\times\big\langle \mu(\mathbf{x})^{-\frac{1}{d+1}}(r(\mathbf{x}))^{\frac{d-1}{d+1}}\big\rangle^{1+\frac{1}{d}}_\mu
\end{split}
\end{equation}
Here $\langle u(\mathbf{x}) \rangle_\mu$ denotes the integral $\int \mu(\mathbf{x})u(\mathbf{x}) d\mathbf{x}$. In the experiment, the value $\langle u(\mathbf{x}) \rangle_\mu$ can be computed by averaging $u$ over the data set distributed according to $\mu$. In Figure \ref{fig:ntk_eig_vals} we compare the theoretical eigenvalue distributions with the numerical NTK distributions for several dimensions $d$ and data set sizes $M$.

\subsection{The loss function}\label{sec:coeff}

We extend now the arguments of Section \ref{sec:singular} to derive the loss asymptotic \eqref{eq:loss_asymp1}. 
We consider two classes of functions $g$ representing the initial error \eqref{eq:g}. 

\paragraph{Scenario 1: a discontinuous function $g$.} We assume that $g$ is supported on a bounded subset $\Omega\subset \mathbb R^d$ with a smooth boundary $\partial \Omega$ so that $g$ has a discontinuity on this boundary but is smooth inside $\Omega$. An obvious example of $g$ is the indicator function of $\Omega.$ In this scenario we obtain
\begin{align}\label{eq:ltdisc0}
L(t)\sim{}&\int_{\partial\Omega}|\Delta g(\mathbf x)|^2 (\mu(\mathbf x)\widetilde{\theta}_{\mathbf x}(\mathbf n))^{-\frac{1}{d+\alpha}}dS\nonumber\\
&\times \tfrac{1}{2\pi}\Gamma(\tfrac{1}{d+\alpha}+1)\cdot(2t)^{-\frac{1}{d+\alpha}}.
\end{align}
Here, $\mathbf x$ is the point on the surface $\partial \Omega$, $\mathbf n$ is the unit normal to $\partial \Omega$, and $\Delta g(\mathbf x)$ is the size of the jump of $g$ at $\mathbf x$, given by the limit of $g(\mathbf y)$ as $\mathbf y$ approaches $\mathbf x$ from inside $\Omega.$ The eigenvalue, coefficient, and loss exponents in this scenario are, respectively,
\begin{equation}
    \kappa = \tfrac{1}{d},\quad\nu=1+\tfrac{\alpha}{d},\quad\xi=\tfrac{\kappa}{\nu}=\tfrac{1}{d+\alpha}.
\end{equation}

\paragraph{Scenario 2: $g$ generated by a Gaussian process.}
Suppose that $g$ is a realization of a Gaussian process with a covariance matrix $\Sigma(\mathbf x,\mathbf x')=\langle g(\mathbf x)g(\mathbf x') \rangle$ and that $\Sigma$ has a homogeneous singularity $\zeta_{\mathbf x}$ of degree $\beta$ at the diagonal (in the sense of Eqs. \eqref{eq:theta10},\eqref{eq:theta20}). In this scenario we find 
\begin{align}\label{eq:ltgauss0}
    L(t)
    \sim{}&\int_{\mathbb R^{d}}\int_{|\mathbf n|=1}\widetilde{\zeta}_{\mathbf x}(\mathbf n)(\mu(\mathbf x)\widetilde{\theta}_{\mathbf x}(\mathbf n))^{-\frac{\beta}{d+\alpha}}d\mathbf xdS\nonumber\\
    &\times \tfrac{1}{2(2\pi)^d\beta}\Gamma(\tfrac{\beta}{d+\alpha}+1)\cdot  (2t)^{-\frac{\beta}{d+\alpha}}.
\end{align}
Here, $\widetilde{\zeta}_{\mathbf x}$ is the Fourier transform of ${\zeta}_{\mathbf x}$ defined as in Eq. \eqref{eq:sing_Fourier}. The eigenvalue, coefficient, and loss exponents in this scenario are, respectively,
\begin{equation}
    \kappa = \tfrac{\beta}{d},\quad\nu=1+\tfrac{\alpha}{d},\quad\xi=\tfrac{\kappa}{\nu}=\tfrac{\beta}{d+\alpha}.
\end{equation}

In our experiments we model GP by a large network in the NTK regime. The corresponding covariance is analyzed in Section \ref{sec:eigvalues_asym_base} and has the singularity degree $\beta=3$. 

We provide the full derivations of Eqs.~\eqref{eq:ltdisc0} and \eqref{eq:ltgauss0} in Section \ref{sec:SM_coeffs} of SM, and sketch now the main ideas. Our general strategy is to complement the localized eigenvalue analysis of Section \ref{sec:singular} by the analysis of expansion coefficients. Note, however, that the simple approximation of $\ta$ by a direct sum of localized operators $\ta|_{L^2(\Omega_m)}$ (as performed in \cite{Birman_1970} and sketched in Section \ref{sec:singular}) is ``too rough'' for the study of expansion coefficients (due to a stronger effect of boundary conditions in each $\Omega_m$). Accordingly, we replace it by the short-time Fourier transform. The initial state $g$ is expanded as  
\begin{align*}
    F(\mathbf y, \mathbf k)={}&(2\pi)^{-d/2}\int_{\mathbb R^d} g(\mathbf x)\omega(\mathbf x-\mathbf y)e^{-i\mathbf k\cdot\mathbf x}d\mathbf x,
    \\
    g(\mathbf x)={}&(2\pi)^{-d/2}\int_{\mathbb R^d}\int_{\mathbb R^d} F(\mathbf y,\mathbf k)\omega(\mathbf x-\mathbf y)e^{i\mathbf k\cdot\mathbf x}d\mathbf yd\mathbf k,
\end{align*}   
where $\omega$ is a  window function such that $\int\omega^2=1$. The coefficient $F(\mathbf y, \mathbf k)$ describes the component of $g$ having the wave number $\mathbf k$ and localized at the point $\mathbf y.$ Then at large $t$, using the stationary phase method,
\begin{align}
    & g_t(\mathbf x)= e^{-t\ta}g(\mathbf x)\\
    & \sim
   (2\pi)^{-\frac{d}{2}}\int_{\mathbb R^d}\int_{\mathbb R^d} F(\mathbf y,\mathbf k) e^{-t\mu_{\mathbf x}\widetilde{\theta}_{\mathbf x}(\mathbf k)}\omega(\mathbf x-\mathbf y)e^{i\mathbf k\cdot\mathbf x}d\mathbf yd\mathbf k.\nonumber
\end{align}
The leading contribution to this integral comes from large $\mathbf k$. For such $\mathbf k,$ we can write  the coefficients $F(\mathbf y,\mathbf k)$ in an asymptotic form, primarily determined by the singularities of $g$. By integrating out $\mathbf y$ and $\mathbf k$, we then arrive at the desired Eqs.~\eqref{eq:ltdisc0} and \eqref{eq:ltgauss0}.

The above argument establishes the loss asymptotic \eqref{eq:loss_asymp1} while bypassing the computation of the asymptotic \eqref{tail asym} of the expansion coefficients aligned with the sorted ``global'' eigenvalues. This latter asymptotic (including the  coefficient $K$) can be found by a similar computation or deduced using Eq. \eqref{eq:loss_asymp}, see Section  \ref{sec:q}.  

\section{Extensions}\label{sec:extensions}
The power law asymptotic of eigenvalues obtained in section \ref{sec:eigvalues_asym_base} was based on the analysis of diagonal singularity of the NTK in the setting of ReLU activation, shallow depth 2, and the NTK regime.  
We argue now that our general approach is not restricted to this narrow setting. To show this, we separately consider three modification of the network from Section \ref{sec:eigvalues_asym_base}. In this section we mostly describe final results, with the derivations described in Section \ref{sec:SM_extension} of SM.

\subsection{Activations of different smoothness}\label{sec:diffsmooth}
\begin{figure}[t]
    \centering
    \includegraphics[width=0.48\textwidth, clip, trim= 5mm 5mm 0 5mm]{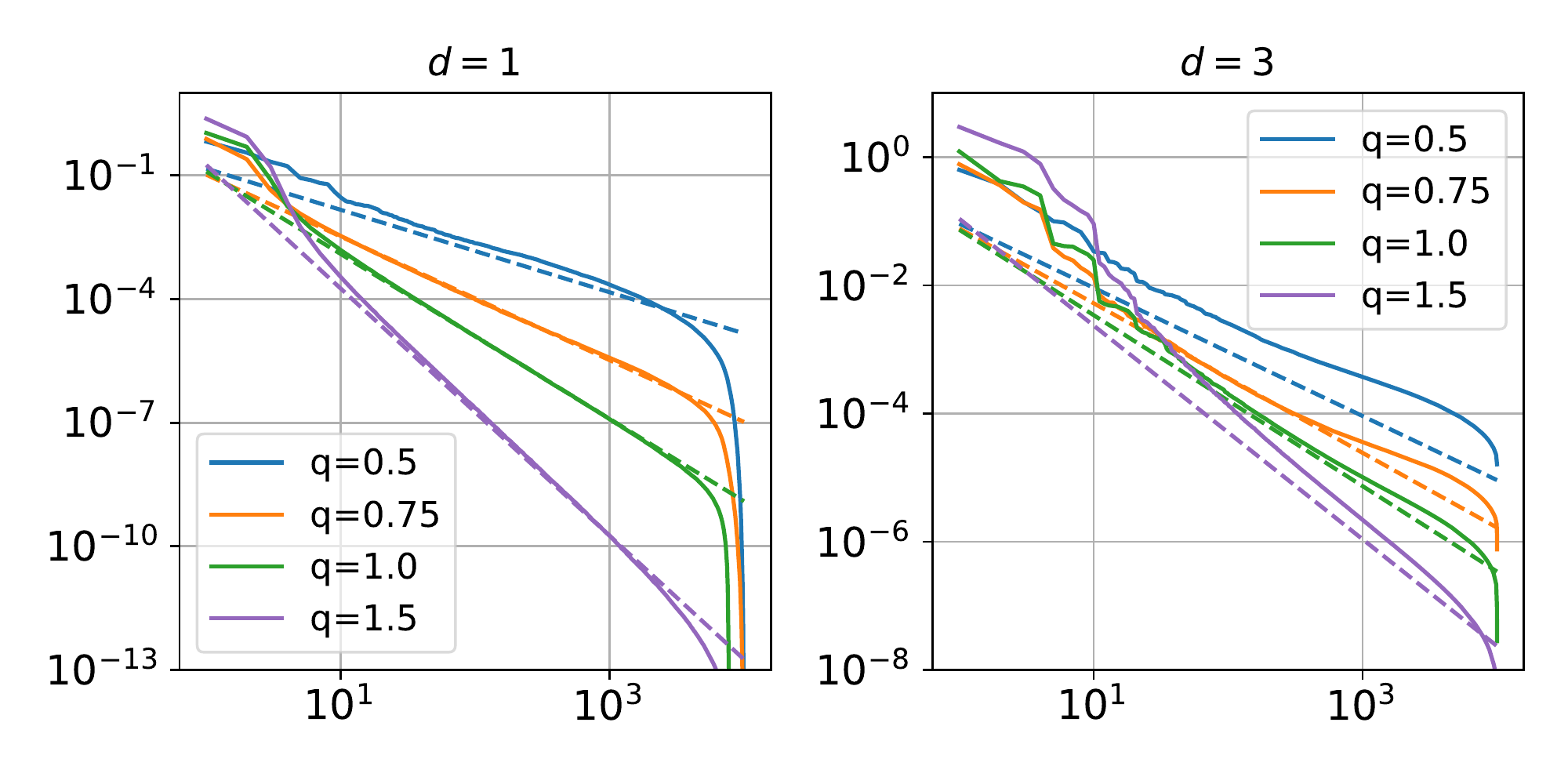}
    \caption{NTK eigenvalues for activation functions $\phi(z) = (z)_+^{\; q}$ with varying smoothness $q$. The theoretical distributions (dashed lines) have the exponents $\nu_q=1+\tfrac{2q-1}{d}$ (see Section \ref{sec:diffsmooth}).} 
  \label{fig:relu_smoothness}
\end{figure}

Let's consider a shallow network in NTK regime with activation function $\phi_q(z)=(z)_+^{\;q}, \; q>0$ (``a ReLU with the altered smoothness $q$'').
Similarly to the ReLU case, one can show that NTK in current setting has a singularity on the diagonal for all values of $q$ except half-integers $q=\tfrac{1}{2},\tfrac{3}{2},\ldots$. The leading singular term is
\begin{align}
    & \label{ReLU_q ntk}
    \Theta_q(\mathbf{x},\mathbf{x'}) = \frac{\sigma_w^2}{2\pi} r^q r'^q a_q \varphi^{2q-1}  \\
    &\label{eq:a_q}
    a_q = \frac{\Gamma^2(q)\Gamma(\tfrac{1}{2}-q)}{\sqrt{\pi}2^{q}}
\end{align}
Here $r,r',\varphi$ are the same as in section \ref{sec:asymp}. 

The singularity of NTK with degree $\alpha=2q-1$ implies the eigenvalue power law asymptotic \eqref{lambda asym} with the exponent $\nu_q=1 + \tfrac{2q-1}{d}$. Thus, the singularity degree in the NTK is determined by the singularity degree of the activation function. The coefficient $\Lambda_q$ can be explicitly computed as
\begin{equation}\label{eq:lambdaq}
    \begin{split}
    \Lambda_q ={}& \sigma_w^{\alpha+2} \sigma_b^{\tfrac{\alpha}{d}} q^2 (2\pi)^{d+q-2} \frac{\Gamma\big(\tfrac{d+\alpha}{2}\big) \Gamma^2(q)}{\big(\Gamma(\frac{d}{2}+1)\big)^{\frac{d+\alpha}{d}}}  \\
    &\times\left\langle \mu(\mathbf{x})^{-\frac{\alpha}{d+\alpha}}r(\mathbf{x})^{\frac{2d-\alpha d - \alpha}{d+\alpha}}\right\rangle^{\frac{d+\alpha}{d}}_\mu
\end{split}
\end{equation}

In Figure \ref{fig:relu_smoothness} we compare the theoretical and numerical eigenvalue distributions for several values of $d$ and $q$.

\subsection{Deep networks}\label{sec:deep}

\begin{figure}[t]
    \centering
    \includegraphics[width=0.48\textwidth, clip, trim= 5mm 8mm 0 5mm]{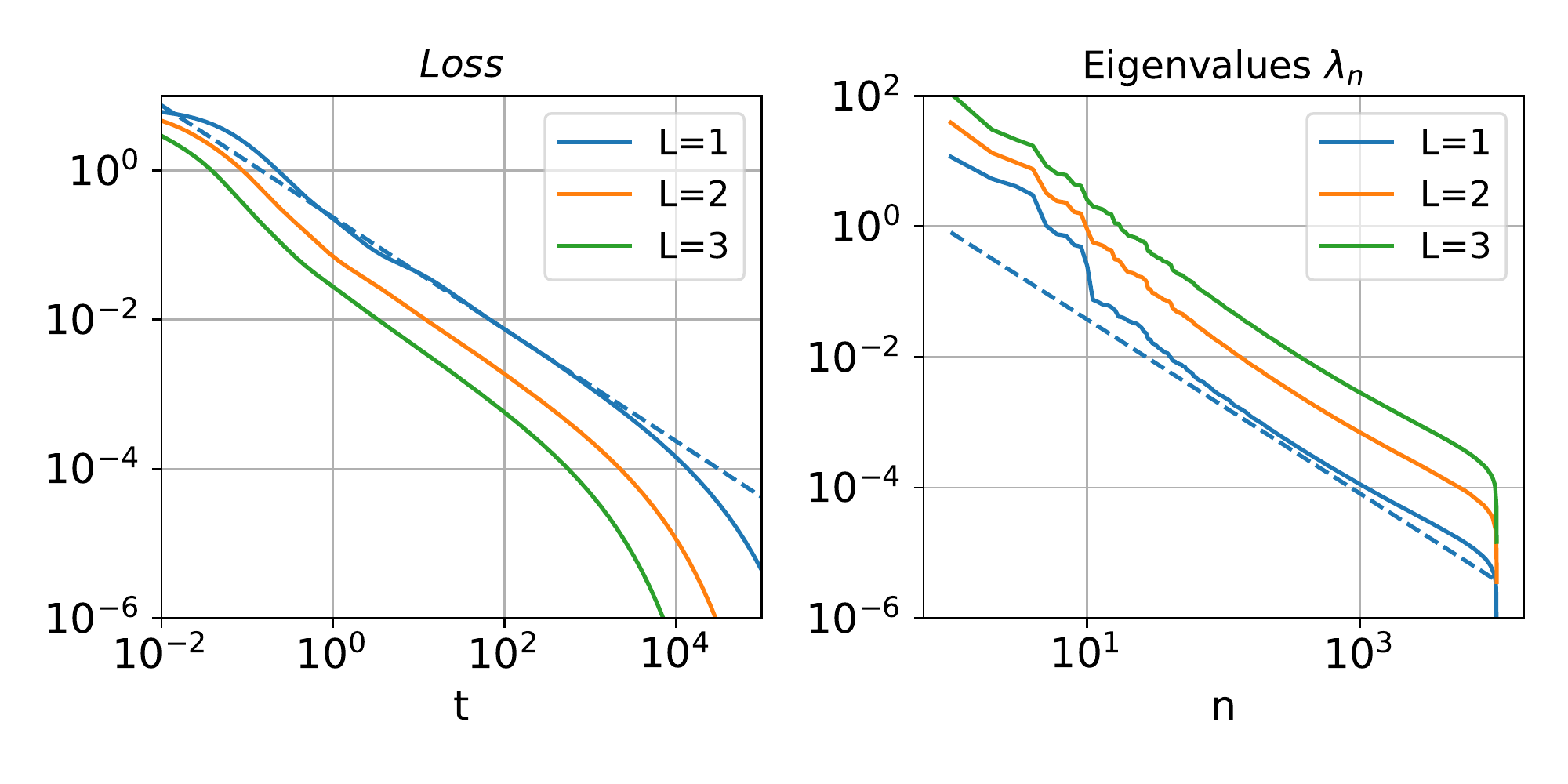}
    \caption{Loss evolution and NTK eigenvalue distribution for networks with varying number of hidden layers $L$ trained on 3-dimensional data and target generated by network GP ($\beta=3$). In agreement with theory, the numerical results (solid lines) show the scaling exponents $\xi=\tfrac{3}{d+1}$ and $\nu=1+\tfrac{1}{d}$  for all depths $L$ (see Section \ref{sec:deep}). The theoretical predictions (dashed lines) are only shown for $L=1$ (for other $L$ the theoretical lines would have the same slopes but different positions determined by the coefficients $C$ and $\Lambda$; due to computation complexity, we have found $C$ and $\Lambda$ only for $L=1$).} 
  \label{fig:deep_net}
\end{figure}

We consider now a network of arbitrary depth $L>2$ in the NTK regime and with the ReLU activation function. Similarly to the shallow ReLU case, the angles $\varphi_l$ (see Eq.~\eqref{eq:covariance_param0}) are singular on the diagonal: $\varphi_l \propto |\mathbf{x}-\mathbf{x'}|$. Using relations \eqref{eq:recur0} one can obtain recursive relations for $\varphi_l,r_l$ and finally for the singular part of NTK $\Theta^{(l)}_{\text{sing}}$
\begin{equation}\label{eq:NTK_sing_recur}
    \begin{split}
    & \Theta^{(l+1)}_{\text{diag}} = r_{l+1}^2 + \frac{\sigma_w^2}{2} \Theta^{(l)}_{\text{diag}} \\
    & \Theta^{(l+1)}_{\text{sing}} = -\frac{1}{2\pi}\Theta^{(l)}_{\text{diag}} \varphi_l + \frac{\sigma_w^2}{2} \Theta^{(l)}_{\text{sing}}
    \end{split}
\end{equation}
Here $\Theta^{(l+1)}_{\text{diag}}$ is the value of NTK on the diagonal. Since $\Theta^{(1)}_{\text{sing}}=0$, we can see from \eqref{eq:NTK_sing_recur} that $\Theta^{(L)}_{\text{sing}}$ is a weighted sum of $-\varphi_l$ for $l=1,\ldots,L-1$. Thus, the singularity degree of the NTK is $\alpha=1$ and the eigenvalue power law \eqref{lambda asym} holds with $\nu=1+\tfrac{1}{d}$. However, obtaining explicit formula for $\Lambda$ is harder for this case and we leave it for future work.  In Figure \ref{fig:deep_net} we compare the theoretical and numerical eigenvalue distributions for NTK's of deep networks.

\subsection{MF regime}\label{sec:mf}
\begin{figure}[tb]
    \centering
    \includegraphics[width=0.48\textwidth,clip,trim= 5mm 8mm 0 5mm ]{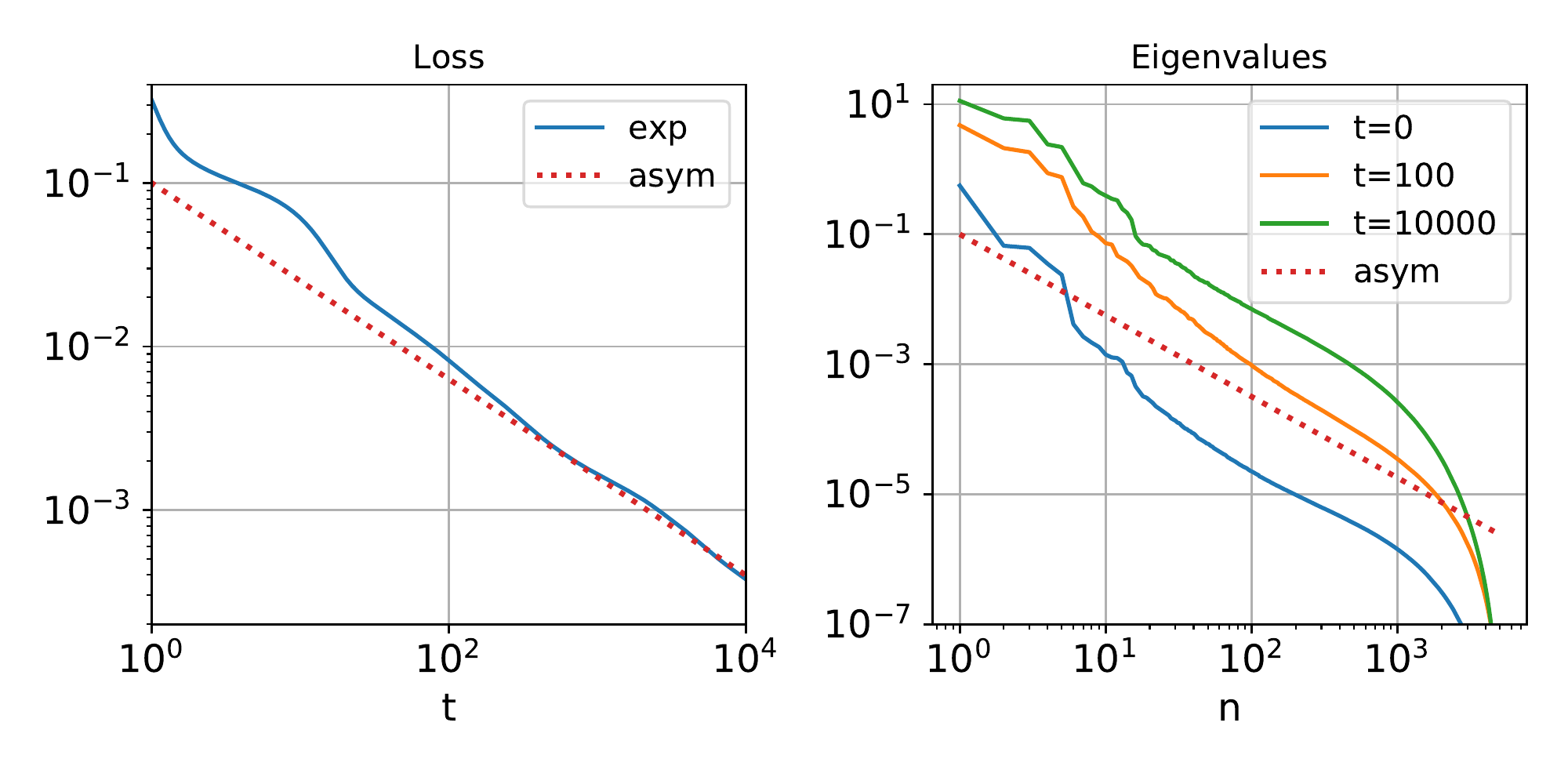}
    \caption{Loss dynamic and eigenvalue distribution at different moments of time for a network in the MF regime (Section \ref{sec:mf}). The network approximates GP on inputs with dimension $d=4$. \textbf{Left:} comparison of the theoretical and experimental loss evolution. The theoretical exponent $\xi=\tfrac{\beta}{d+1}=\tfrac{3}{5}$.  \textbf{Right:} the theoretical eigenvalue distribution and its experimental counterparts at different times. The theoretical exponent is $\nu=\tfrac{d+1}{d}=\tfrac{5}{4}$.} 
  \label{fig:mf_net}
\end{figure}

In this section we consider a shallow network in the MF regime and with the ReLU activation function, with the NTK given by \eqref{eq:MF_NTK0}. The neural tangent kernel \eqref{ReLU ntk} of the shallow NTK network  can be obtained from \eqref{eq:MF_NTK0} with the distribution $p_\infty(c,\mathbf{w},b)$ taken as a product of Gaussians  for each variable. Therefore, neural tangent kernels in MF regime represent a broader class of kernels, containing our basic example \eqref{ReLU ntk}. It turns out that the diagonal singularity is present for all sufficiently smooth and quickly decaying distributions $p_\infty$ of this broader class:
\begin{equation*}
    \Theta_{\text{sing}}(\mathbf{x},\mathbf{x'}) = -\frac{\sqrt{1+|\mathbf{x}|^2}}{2}\int\limits_{\mathbf{w}\cdot\mathbf{x}=-b} \big|\mathbf{w}\cdot\delta\mathbf{x}\big| dp_{\infty,2}(\mathbf{w},b)
\end{equation*}
Here $\delta \mathbf{x}=\mathbf{x}-\mathbf{x'}$ and $p_{\infty,2}$ is the second moment of the distribution $p_\infty$ w.r.t. the variable $c$, i.e. $p_{\infty,2}(\mathbf{w},b)=\int c^2 p_{\infty}(c,\mathbf{w},b)dc $. Thus, the eigenvalue power law exponent is the same as in the NTK regime: $\nu=1+\tfrac{1}{d}$. 

In Figure \ref{fig:mf_net} we show the loss dynamic and the eigenvalue distribution at different moments of time for a network in the MF regime. Since the neural tangent kernel significantly changes during the training of the MF network,  picking NTK's at different moments of training provides sufficiently distinct and general kernels of the form \eqref{eq:MF_NTK0}.

\bibliography{universal_scaling}
\bibliographystyle{icml2021}

\appendix
\onecolumn

\numberwithin{equation}{section}

\DeclarePairedDelimiter\floor{\lfloor}{\rfloor}
\def\lf{\left\lfloor}   
\def\rf{\right\rfloor}

\begin{center}
{\Large Supplementary material}
\end{center}

\section{Details of experiments}\label{sec:exp_det}
In this work we have two types of experiments. Both types operate with dataset consisting of $M$ data samples drawn from some distribution $\mu(\mathbf{x})$. 

In our experiments we focus on distributions $\mu$ different from simple standard distributions such us  spherical Gaussian or uniform in the cube. Highly symmetric distributions such as the spherical Gaussian make the whole problem analytically solvable (see, for example, \cite{yang2019a}). In contrast, our theory do not rely on the symmetry, and we test it on distributions $\mu$ without any symmetry. In all our experiments, the distribution $\mu(\mathbf{x})$ is constructed as follows: we randomly choose $n_g$ points in the cube $[-1,1]^d$ and consider $n_g$ symmetric Gaussian distributions with centers in these points and standard deviation $\sigma$; then $\mu(\mathbf{x})$ is defined as the average of these $n_g$ Gaussian distributions. Although each separate Gaussian distribution is symmetric, the average of $n_g > d$ randomly located Gaussian distributions almost surely removes all symmetries w.r.t. orthogonal transformations of $\mathbb{R}^d$. The typical values used in our experiments are $n_g=8$ and $\sigma=0.5.$        

In the first type of experiments we analytically calculate the NTK using, for example, expression \eqref{ReLU ntk}, and then numerically diagonalize the corresponding matrix and decompose a target function over its eigenvectors. Then the linear evolution \eqref{eq:dfdta} can be easily computed in obtained eigenbasis. The center and right parts of figure \ref{fig:loss_dyn} as well as figures \ref{fig:ntk_eig_vals},\ref{fig:relu_smoothness},\ref{fig:deep_net} correspond to this type of experiments. Thus, it can be considered as an experiment with infinitely wide network, but a finite dataset. In all such experiments we take the largest $M$ possible, which is bounded by $O(M^3)$ time cost of numerical diagonalization and $O(M^2)$ memory cost of storing the NTK. The typical value used in experiments is $M=10000$. To calculate the constant $\Lambda$ in the eigenvalue asymptotics \eqref{eq:Lambda} and \eqref{eq:lambdaq} we draw another, rather big, set of points from the distribution $\mu(\mathbf{x})$ and use this dataset for Monte Carlo estimation of the $\langle \cdot \rangle_\mu$ averages in \eqref{eq:Lambda} and \eqref{eq:lambdaq}.                

In the second type of experiments we initialize and train actual wide network (typical width $N=3000$). The left part of figure \ref{fig:loss_dyn} and figure \ref{fig:mf_net} correspond to this type of experiments. To reach large values of time $t$ we choose the learning rate $\eta$ close to its critical value $\eta_c = \tfrac{2}{\lambda_0}$, above which the dynamic in the $0$'th eigenspace start to diverge exponentially and the network leaves the regime of approximately constant NTK. The network in the MF regime adapts to learning rates higher than critical one at initialization, but adaptation resource is limited. Overall, experiments of this type test our theoretical predictions for roughly practical sizes of networks and datasets. 

In the experiments we considered two types of target functions. The first is a draw from a Gaussian process, which we model by a very wide $N=10^6$ shallow network with NTK parametrization. Thus, the covariance of GP is given by \eqref{ReLU covariance}. To calculate the coefficient $C$ in the loss asymptotic \eqref{eq:loss_asymp1} we use Eq.~\eqref{eq:ltgauss0} with the sphere integral taken analytically as shown in  \eqref{eq:sphere_int}. The second type of target is an indicator of a ball of some radius $r$. It corresponds to two-class classification task with first class located in $|\mathbf{x}|<r$ and the second class in $|\mathbf{x}|>r$. To calculate coefficient $C$ we sampled points uniformly on a sphere $|\mathbf{x}|=r$ and used them to calculate a Monte Carlo estimate of the integral in \eqref{eq:ltdisc0}. In principle, one can choose classes with more sophisticated separation boundary, but numerical calculation of the integral in \eqref{eq:ltdisc0} will be more complicated. Note also that even if the target has a spherical symmetry, the whole problem does not, because we use a non-symmetric $\mu$.

\newpage
\section{Derivation of the loss asymptotic from the asymptotics of the eigenvalues $\lambda_n$ and the expansion coefficients $c_n$}\label{sec:genlossformula}

In this section we prove the loss asymptotic \eqref{eq:loss_asymp}. This result is established under assumption of power law asymptotics \eqref{lambda asym}, \eqref{tail asym} for the eigenvalues $\lambda_n$ and partial sums of coefficients $s_n=\sum_{k \geq n} |c_k|^2$, i.e.
\begin{equation}\label{eq:base_asym}
    \begin{split}
        &\lambda_n \sim \Lambda n^{-\nu}, \\
        &s_n \sim K n^{-\kappa}.
    \end{split}
\end{equation}
Here the asymptotic similarity sign $\sim$ denotes $a_n\sim b_n \iff a_n=b_n(1 + o(1))$.

\begin{ther}
Under the assumptions \eqref{eq:base_asym} on the asymptotic of eigenvalues ad coefficients, the loss $L(t)=\tfrac{1}{2}\sum_n e^{-2\lambda_n t}|c_n|^2$ has the asymptotic
\begin{equation}
\label{eq:loss_asym}
    L(t) \sim \frac{K}{2} \Gamma\left(\frac{\kappa}{\nu} + 1\right) \left( 2 \Lambda  t\right)^{-\frac{\kappa}{\nu}}
\end{equation}

\end{ther}
\begin{proof}

The constant $\Lambda$ enters the loss only in combination with $t$. Thus, by rescaling time and noticing that the loss is proportional to $K$, it is sufficient to consider the case $K=2$ and $2\Lambda=1$. In other words, we have to prove
\begin{equation}
    \sum\limits_{n=0}^{\infty} e^{-tn^{-\nu}(1+u_n)} (s_n-s_{n+1}) \sim \Gamma\left(\frac{\kappa}{\nu} + 1\right) t^{-\frac{\kappa}{\nu}}, \qquad \text{with} \quad s_n = n^{-\kappa}(1+v_n) 
\end{equation}

Here $\lim_{n\rightarrow\infty} u_n = \lim_{n\rightarrow\infty} v_n = 0$ due to asymptotic \eqref{eq:base_asym}. 

The idea of the proof is that in the region of small $n$ the sum can be neglected due to exponential factor $e^{-t\#}$, while in the region of large $n$ the sum can be replaced by the integral, since the sum argument slowly depend on $n$ in this region. In fact both regions greatly overlap, and we can find a common point $n_t$ inside both regions. Such common point can be taken as $n_t = \lf t^\beta \rf$ with any $\beta$ from the interval $(\tfrac{1}{\nu+1},\tfrac{1}{\nu})$ (the reason will be seen later). Let's denote
\begin{equation}
    u_t \equiv \sup\limits_{n \geq n_t} |u_n|, \qquad v_t \equiv \sup\limits_{n \geq n_t} |v_n|
\end{equation}
Since $u_n, v_n \rightarrow 0$ at $n\rightarrow \infty$ and $n_t \rightarrow \infty$ at $t\rightarrow \infty$, we have $u_t, v_t \rightarrow 0$ at $t\rightarrow \infty$. The strategy of the proof is first to bound the sum for $n \leq n_t$, then calculate the sum for $n>n_t$ with $u_n=v_n=0$, and finally add corrections from $u_n$ and $v_n$. 

The sum over $n \leq n_t$ is bounded as
\begin{equation}
\begin{split}
\sum\limits_{n=0}^{n_t} e^{-2t\lambda_n} (s_n-s_{n+1}) \leq e^{-2t\lambda_{n_t}} \sum\limits_{n=0}^{n_t} (s_n-s_{n+1}) \\ 
\leq \exp\left(-t n_t^{-\nu}(1-u_t)\right) s_0 \leq \exp\left(-t^{1-\beta \nu}(1-u_t)\right) s_0
\end{split}
\end{equation}
Since $\beta < \tfrac{1}{\nu}$ we have $1-\beta \nu > 0$ and the sum goes to $0$ exponentially fast as $t\rightarrow \infty$. 

Now we calculate the sum over $n > n_t$ with $u_n=v_n=0$. Due to convexity of $f(x)=x^{-\alpha}$ for all $\alpha>0$ we have the bounds  
\begin{equation}\label{eq:bounds_1}
    \kappa (n+1)^{-\kappa-1} \leq s_n - s_{n+1} \leq \kappa n^{-\kappa-1}
\end{equation}
Then we approximate the sum with the integral as
\begin{equation}\label{eq:int_calc}
\begin{split}
&\sum\limits_{n>n_t}e^{-tn^{-\nu}}\big( n^{-\kappa} - (n+1)^{-\kappa}\big) \stackrel{(1)}{=} \int\limits_{n_t}^{\infty} e^{-tx^{-\nu}} \big(1 + O(t x^{-\nu-1})\big) \kappa x^{-\kappa-1}\big(1 + O(x^{-1})\big) dx \\
& \stackrel{(2)}{=} \int\limits_0^{t n_t^{-\nu}} e^{-z} \kappa \left(\frac{t}{z}\right)^{-\tfrac{\kappa+1}{\nu}} \frac{1}{t \nu} \left(\frac{t}{z}\right)^{1+\tfrac{1}{\nu}} dz\Big(1 + O(tn_t^{-\nu-1}) + O(n_t^{-1})\Big) \\
& = \frac{\kappa}{\nu} t^{-\tfrac{\kappa}{\nu}}\left(\int\limits_0^{\infty} z^{\tfrac{\kappa}{\nu}-1} e^{-z} dz + O\Big(\exp\big(-t n_t^{-\nu}\big)\big[t n_t^{-\nu}\big]^{\tfrac{\kappa}{\nu}-1}\Big) + O(tn_t^{-\nu-1}) + O(n_t^{-1}) \right)\\
&\stackrel{(3)}{=} \Gamma\Big(\frac{\kappa}{\nu}+1\Big)t^{-\tfrac{\kappa}{\nu}}\Big(1+O\Big(\exp\big(-t^{1-\beta\nu}\big)\big[t^{(1-\beta\nu)}\big]^{\tfrac{\kappa}{\nu}-1}\Big)+O\big(t^{1-\beta(\nu+1)}\big)+O\big(t^{-\beta}\big)\Big)
\end{split}
\end{equation}
Here in (1) we used \eqref{eq:bounds_1} to estimate the difference $s_n-s_{n+1}$, and then first order Taylor expansion to estimate value of integrated function at non-integer points. In (2) we made a change of variables $z=tx^{-\nu}$ and estimated "big O" terms using minimum value $x_{\text{min}}=n_t$. In (3) used the definition of Gamma function and recursive relation $z\Gamma(z)=\Gamma(z+1)$, and substituted $n_t$. $\beta \in (\tfrac{1}{\nu+1}, \tfrac{1}{\nu})$ implies that $1-\beta\nu>0$ and $1-\beta(\nu+1)<0$, therefore all "big O" terms go to $0$ as $t\rightarrow \infty$.

The last step is to include back $u_n$ and $v_n$ into sum over $n>n_t$. To include $u_n$ we use that $n^{-\nu}(1 - u_t)\leq 2\lambda_n \leq n^{-\nu}(1 + u_t)$. Then we make lower/upper bounds for the sum with $u_n$ by using the result of calculation \eqref{eq:int_calc} with substitution $t\rightarrow t(1 \pm u_t)$. Finally, we bound contribution from $v_n$ as
\begin{equation}
\begin{split}
    &\sum\limits_{n>n_t} e^{-2t\lambda_n}(v_n n^{-\kappa}-v_{n+1}(n+1)^{-\kappa}) \\
    &\stackrel{(1)}{=} \sum\limits_{n>n_t} \big(e^{-2t\lambda_n} - e^{-2t\lambda_{n-1}}\big) v_n n^{-\kappa} + e^{-2t\lambda_{n_t}} v_{n_t+1} (n_t+1)^{-\kappa} \\
    & \stackrel{(2)}{\leq} v_t \sum\limits_{n>n_t} e^{-2t\lambda_n}( n^{-\kappa}-(n+1)^{-\kappa}) + e^{-2t\lambda_{n_t}} (v_{n_t+1} - v_t) (n_t+1)^{-\kappa} 
\end{split}
\end{equation}
Here in (1) we regrouped the terms in the sum, while in (2) we used $v_n \stackrel{n>n_t}{\leq} v_t$ and regrouped summation terms back into original form. The second term in the last line is negative and the first term is $v_t$ times the result of \eqref{eq:int_calc}. The final expression with all contributions is
\begin{equation}
    \begin{split}
        &\sum\limits_{n=0}^{\infty} e^{-2t\lambda_n} (s_n-s_{n+1}) = \Gamma\Big(\frac{\kappa}{\nu}+1\Big)t^{-\tfrac{\kappa}{\nu}} \times \\
        & \times \Big(1+O\Big(\exp\big(-t^{1-\beta\nu}\big)\big[t^{(1-\beta\nu)}\big]^{\max(\tfrac{\kappa}{\nu}-1,0)}\Big)+O\big(t^{1-\beta(\nu+1)}\big) + O(u_t) + O(v_t)\Big)
    \end{split}
\end{equation}
Here all "big O" terms vanish in the limit $t\rightarrow \infty$ and we obtain desired answer.
\end{proof}

\newpage
\section{Calculation of $\gamma_\mathbf{x}$}\label{sec:gammax}
According to formula \eqref{eq:nu1ad} from the main text, to calculate the coefficient $\Lambda$ in the eigenvalue asymptotic one needs to find the volume $\gamma_\mathbf{x}$ of the region defined using the Fourier transform $\widetilde{\theta}_\mathbf{x}(\mathbf{k})$ of the homogeneous singularity $\theta_\mathbf{x}(\mathbf{z})$. These quantities are defined as 
\begin{align}
    &\label{eq:Fourier_def}
    \widetilde \theta_{\mathbf x}(\mathbf k)=\lim_{r\to\infty}\int_{|\mathbf x|<r}(1-\tfrac{|\mathbf x|^2}{r^2})^c \theta(\mathbf z)e^{-i\mathbf k\cdot\mathbf z}d\mathbf z \\
    &\label{eq:gamma_def}
    \gamma_{\mathbf x}=(2\pi)^{-d}|\{\mathbf k\in\mathbb R^d: \widetilde \theta_{\mathbf x}(\mathbf k)>1\}|        
\end{align}
In this section we calculate $\gamma_x$ in the case of singularities $\theta_\mathbf{x}(\mathbf{z})=|\mathbf{z}|^\alpha$ and $\theta_\mathbf{x}(\mathbf{z})=\varphi^\alpha(\mathbf{x},\mathbf{x}+\mathbf{z})$ with $\varphi(\mathbf{x},\mathbf{x'})$ defined as in the main text. The latter case is needed for obtaining coefficients $\Lambda$ in Eqs.~\eqref{eq:Lambda} and \eqref{eq:lambdaq}.

\paragraph{Case $\theta_\mathbf{x}(\mathbf{z})=|\mathbf{z}|^\alpha$.} We drop index $\mathbf{x}$ since there is no dependence on it. The homogeneity of $\theta(\mathbf{x})$ with degree $\alpha$ implies the homogeneity of $\widetilde{\theta}(\mathbf{k})$ with degree $-\alpha-d$. This can be seen from definition \eqref{eq:Fourier_def} by making integration variable change $\mathbf{x}\rightarrow c\mathbf{x}, \; c>0$. Then, due to spherical symmetry of $\theta(\mathbf{z})=|\mathbf{z}|^\alpha$, its Fourier transform has a form 
\begin{equation}
    \widetilde{\theta}(\mathbf{k}) = c_{d,\alpha} |\mathbf{k}|^{-d-\alpha}
\end{equation}
To determine the coefficient $c_{d,\alpha}$ we take into account that both $\theta(\mathbf{z})$ and $\widetilde{\theta}(\mathbf{k})$ are generalized functions acting on some test functions $\chi$. We denote the Fourier transform by $F$ and action of generalized functions on test functions by $\langle \cdot,\cdot\rangle$ (not to be confused with averaging in Eqs.~\eqref{eq:Lambda} and \eqref{eq:lambdaq}). Then by taking a test function $\chi(\mathbf{k})=e^{-|\mathbf{k}|^2/2}$ and its Fourier transform $\widetilde{\chi}(\mathbf{z})=(2\pi)^{d/2} e^{-|\mathbf{z}|^2/2}$ we get
\begin{align}
\label{eq:action_1}
\begin{split}
    \langle F(\theta),\chi \rangle &= c_{d,\alpha} \int |\mathbf{k}|^{-d-\alpha} e^{-\tfrac{|\mathbf{k}|^2}{2}} d\mathbf{k}= c_{d,\alpha} S_{d-1} \int\limits_0^\infty k^{-1-\alpha} e^{-\tfrac{k^2}{2}} dk \\
    & = c_{d,\alpha} S_{d-1} 2^{-\tfrac{2+\alpha}{2}} \int\limits_0^{\infty} e^{-x} x^{-\tfrac{2+\alpha}{2}} dx = c_{d,\alpha} S_{d-1} 2^{-\tfrac{2+\alpha}{2}} \Gamma\Big(-\frac{\alpha}{2}\Big) 
\end{split} \\
\label{eq:action_2}
\begin{split}
    \langle \theta,F(\chi) \rangle &= (2\pi)^{\tfrac{d}{2}}\int |\mathbf{z}|^{\alpha} e^{-\tfrac{|\mathbf{z}|^2}{2}} d\mathbf{z}= (2\pi)^{\tfrac{d}{2}} S_{d-1} \int\limits_0^\infty z^{d-1+\alpha} e^{-\tfrac{z^2}{2}} dz \\
    & = (2\pi)^{\tfrac{d}{2}} S_{d-1} 2^{\tfrac{d-2+\alpha}{2}} \int\limits_0^{\infty} e^{-x} x^{\tfrac{d-2+\alpha}{2}} dx = (2\pi)^{\tfrac{d}{2}} S_{d-1} 2^{\tfrac{d-2+\alpha}{2}} \Gamma\Big(\frac{d+\alpha}{2}\Big) 
\end{split}
\end{align}
Here in both calculations we first integrated over $d-1$-dimensional sphere with the area $S_{d-1}$. Then we changed variables so that the integral can be expressed in terms of Gamma function. Since $\langle F(\theta),\chi \rangle=\langle \theta,F(\chi) \rangle$ we compare expressions \eqref{eq:action_1} and \eqref{eq:action_2} and find
\begin{equation}
    c_{d,\alpha} = 2^{d+\alpha} \pi^{\tfrac{d}{2}} \frac{\Gamma\big( \tfrac{d+\alpha}{2}\big)}{\Gamma\big( -\tfrac{\alpha}{2}\big)}
\end{equation}
To find $\gamma$ we notice that the volume in \eqref{eq:gamma_def} is a ball with radius $c_{d,\alpha}^{1/(d+\alpha)}$. By using a volume of a unit ball in d-dimensional space $B_d = \pi^{d/2} / \Gamma\big(\tfrac{d}{2}+1\big)$ we get
\begin{equation}
    \gamma = \frac{1}{\Gamma\big(\tfrac{d}{2}+1\big)} \left[ \frac{\Gamma\big(\tfrac{d+\alpha}{2}\big)}{\pi^{\tfrac{\alpha}{2}}\Big|\Gamma\big(-\tfrac{\alpha}{2}\big)\Big|}\right]^{\tfrac{d}{d+\alpha}} \equiv \gamma_{d,\alpha}
\end{equation}

\paragraph{Case $\theta_\mathbf{x}(\mathbf{z})=\varphi^\alpha(\mathbf{x},\mathbf{x}+\mathbf{z})$.} To see that angle $\varphi(\mathbf{x},\mathbf{x'})$ has a singularity at $\mathbf{x}=\mathbf{x'}$ we write the scalar product of $\widetilde{\mathbf{x}},\widetilde{\mathbf{x}}'$ as 
\begin{equation} \label{eq:phi1_def}
    \cos \varphi(\mathbf{x},\mathbf{x}') \sqrt{r(\mathbf{x})r(\mathbf{x'})}=\sigma_w^2 \mathbf{x} \cdot \mathbf{x'} + \sigma_b^2 
\end{equation}
with $r^2(\mathbf{x})=\sigma_w^2|\mathbf{x}|^2+\sigma_b^2$ as in the paper. Expanding this expression at $\varphi=0$ and $\delta \mathbf{x} = \mathbf{x}-\mathbf{x}'=0$ we get
\begin{equation}
\label{eq:local_angle}
    \varphi(\mathbf{x},\mathbf{x}') = \sqrt{1 -\frac{\sigma_w^2|\mathbf{x}|^2}{r^2(\mathbf{x})}\cos^2\psi} \; \frac{\sigma_w |\delta \mathbf{x}|}{r(\mathbf{x})} + O\big(|\delta \mathbf{x}|^2\big), 
\end{equation}
which is the expression \eqref{local angle} from the main text. In the asymptotic analysis we need to consider only the leading singular term, therefore the homogeneous singularity has the form
\begin{equation}\label{eq:angle_sing}
    \theta_\mathbf{x}(\mathbf{z}) = a(\mathbf{x})\left(\sqrt{1-b(\mathbf{x}) \frac{z_1^2}{|\mathbf{z}|^2}} \; |\mathbf{z}|\right)^\alpha
\end{equation}
Here $a(\mathbf{x})=(\sigma_w /r(\mathbf{x}))^\alpha$ and $b(\mathbf{x})=\sigma^2_w |\mathbf{x}|^2/r^2(\mathbf{x})$ are introduced for convenience and we also omit $\mathbf{x}$ for the rest of this section. In \eqref{eq:angle_sing} we oriented basis in $\mathbf{z}$ space so that the first axis is parallel to vector $\mathbf{x}$. Now we calculate Fourier transform 
\begin{equation}
\begin{split}
    \widetilde{\theta}(\mathbf{k}) &= \int \theta(\mathbf{k})e^{-i\mathbf{k}\cdot \mathbf{z}} d\mathbf{z} = a \int \Big[(1-b)z_1^2 +z_2^2 + \ldots + z_d^2\Big]^{\frac{\alpha}{2}}e^{-i\mathbf{k}\cdot \mathbf{z}} d\mathbf{z} \\
    & \stackrel{(1)}{=} \frac{a}{\sqrt{1-b}} \int |\mathbf{z}'|^\alpha e^{-i\mathbf{k}'\cdot \mathbf{z}'} d\mathbf{z}' = \frac{a}{\sqrt{1-b}} c_{d,\alpha} |\mathbf{k}'|^{-d-\alpha}
\end{split}
\end{equation}
Here in (1) we changed to variables $\mathbf{z}',\mathbf{k}'$ which are the same as $\mathbf{z},\mathbf{k}$ except the first dimension: $z_1' = z_1\sqrt{1-b}$ and $k_1' = k_1 / \sqrt{1-b}$. In the original $\mathbf{k}$ space the equation $c|\mathbf{k'}|^{-d-\alpha}=1$ defines an ellipsoid obtained from the sphere $c|\mathbf{k}|^{-d-\alpha}=1$ by squeezing the first axis by the factor $\sqrt{1-b}$. This gives us the formula for volume $\gamma$
\begin{equation}
    \gamma = (2\pi)^{-d} B_d \left(\frac{c_{d,\alpha} a}{\sqrt{1-b}}\right)^{\tfrac{d}{d+\alpha}} \sqrt{1-b} = \gamma_{d,\alpha} a^{\tfrac{d}{d+\alpha}} \Big(\sqrt{1-b}\Big)^{\tfrac{\alpha}{d+\alpha}} 
\end{equation}
Restoring $\mathbf{x}$ dependence and using $\sqrt{1-b(\mathbf{x})}=\sigma_b/r(\mathbf{x})$ we get
\begin{equation}\label{eq:gamma_alpha}
    \gamma_{\mathbf{x}} = \gamma_{d,\alpha} \sigma_w^{\tfrac{\alpha d}{d+\alpha}} \sigma_b^{\tfrac{\alpha}{d+\alpha}} r(\mathbf{x})^{-\tfrac{\alpha d+\alpha}{d+\alpha}}
\end{equation}

\paragraph{Case $\theta_\mathbf{x}(\mathbf{z})=A(\mathbf{x}) \varphi^\alpha(\mathbf{x},\mathbf{x}+\mathbf{z})$.} 
This case includes the singularities of shallow network NTK \eqref{eq:NTK_sing} and covariance \eqref{eq:Sigma_sing}, and can be easily obtained from the previous one. Since Fourier transformation \eqref{eq:Fourier_def} and volume calculation \eqref{eq:gamma_def} are performed locally, $A(\mathbf{x})$ is effectively constant in these calculations. Thus $\gamma_\mathbf{x}$ from \eqref{eq:gamma_alpha} is simply multiplied by $|A(\mathbf{x})|^{\tfrac{d}{d+\alpha}}$. The result is
\begin{equation}\label{eq:gamma_alpha2}
    \gamma_{\mathbf{x}} = |A(\mathbf{x})|^{\tfrac{d}{d+\alpha}} \gamma_{d,\alpha} \sigma_w^{\tfrac{\alpha d}{d+\alpha}} \sigma_b^{\tfrac{\alpha}{d+\alpha}} r(\mathbf{x})^{-\tfrac{\alpha d+\alpha}{d+\alpha}}
\end{equation} 

Equation \eqref{eq:Lambda} can be obtained by using $A(\mathbf{x})=-\frac{\sigma_w^2 r(\mathbf{x})^2}{2\pi}$ from \eqref{eq:NTK_sing} and then substituting resulting $\gamma_\mathbf{x}$ in \eqref{eq:nu1ad}.

\newpage
\section{Derivation of the loss asymptotic for singular evolution operators $\ta$ and specific classes of target functions $g$}\label{sec:SM_coeffs}
This section expands the content of Section \ref{sec:coeff} of the main text. Our goal is to derive the explicit leading terms in the asymptotic of the loss $L(t)=\tfrac{1}{2}\|g_t\|^2$ for the evolution $g_t=e^{-t\ta}g$ when $g$ belongs to one of the following two classes:
\begin{enumerate}
    \item  The function $g$ is supported on a bounded subset $\Omega$ with a smooth boundary $\partial \Omega$ so that $g$ has a discontinuity on this boundary but is smooth inside  $\Omega$. An obvious example of $g$ is the indicator function of $\Omega.$
    \item The function $g$ is a realization of a Gaussian process with a particular singular covariance.
\end{enumerate}

Before providing these derivations, let us first recall our general setting. We consider the evolution 
$g_t=e^{-t\ta}g$
governed by the non-negative definite generator   
\begin{equation}\label{eq:wtadf}\widetilde{\mathcal A}g(\mathbf x)=\int_{\mathbb R^d}\mu^{1/2}(\mathbf x)\Theta(\mathbf x,\mathbf x')\mu^{1/2}(\mathbf x')g(\mathbf x')d\mathbf x'\end{equation}
(see Section \ref{sec:loss_asym} of the main text). Here, $\mu$ is the measure corresponding to the data distribution, and $\Theta$ is the kernel associated with the neural network ansatz. This symmetric form of the evolution operator is valid in the representation in which the original functions are multiplied by $\mu^{1/2}$. Accordingly, the function $g$ appearing in the loss formula $L(t)=\tfrac{1}{2}\|e^{-t\ta}g\|^2$ is given by
$$g(\mathbf x)=\mu^{1/2}(\mathbf x)(\widetilde f(\mathbf w(t=0),\mathbf x)-f(\mathbf x)),$$
where $f$ is the function to be fitted by the network, and $\widetilde f(\mathbf w(t=0),\mathbf x)$ is the initial network approximation.\footnote{We remark in passing that we can usually control the initial approximation $\widetilde f$, and in some cases we can ensure that the contribution of  $\mu^{1/2}\widetilde f$ to $g$ is small compared to $\mu^{1/2}f$. In such cases,  one can assume $g\approx -\mu^{1/2}f$.}

In scenario 1 above  -- a function $g$ supported on a domain $\Omega$ and discontinuous on the boundary $\partial\Omega$ -- we will show that the large-$t$ asymptotic of the loss is given by
\begin{align}\label{eq:ltdisc}
L(t)\sim{}&\frac{1}{2\pi}\Gamma\left(\frac{1}{d+\alpha} + 1\right)\int_{\partial\Omega}|\Delta g(\mathbf x)|^2 (\mu(\mathbf x)\widetilde{\theta}_{\mathbf x}(\mathbf n))^{-\frac{1}{d+\alpha}}dS \cdot (2t)^{-\frac{1}{d+\alpha}},
\end{align}
where integration is performed over the boundary, $\mathbf x\in\partial\Omega$ is the respective boundary point, $\mathbf n$ is the respective unit normal to the boundary, $\Delta g(\mathbf x)$ is the value of discontinuity of $g$ at $\mathbf x,$ and $\widetilde{\theta}_{\mathbf x}$ is the Fourier transform of the homogeneous singularity of the kernel $\Theta$ at $\mathbf x=\mathbf x'$.

In scenario 2 -- a function $g$ generated by a Gaussian process with a homogeneous diagonal singularity of degree $\beta$ -- we will show that the large-$t$ asymptotic of the loss is given by
\begin{align}\label{eq:ltgauss}
    L(t)
    \sim{}&\frac{1}{2(2\pi)^d\beta}\Gamma\Big(\frac{\beta}{d+\alpha}+1\Big)\int_{\mathbb R^{d}}\int_{|\mathbf n|=1}\widetilde{\zeta}_{\mathbf x}(\mathbf n)(\mu(\mathbf x)\widetilde{\theta}_{\mathbf x}(\mathbf n))^{-\frac{\beta}{d+\alpha}}d\mathbf xdS\cdot  (2t)^{-\frac{\beta}{d+\alpha}},
\end{align}
where $\widetilde{\zeta}_{\mathbf x}$ is the Fourier transform of the  diagonal singularity.

The general approach in obtaining these power law asymptotics is to expand $g$ over the approximate, spatially localized eigenvectors of the evolution operator $\ta$. One can consider two slightly different versions of this approach. In one version we first find the asymptotic \eqref{tail asym} of the cumulative distribution of the expansion coefficients, and then find the asymptotic of the loss using formula \eqref{eq:loss_asymp} of the main text. In the other version, we bypass the computation of Eq.~ \eqref{tail asym}, and find the asymptotic of $L(t)$ directly. (In fact, this version also uses the asymptotic relation \eqref{eq:loss_asymp} of the main text, but applies it not to the full set of eigenvalues, but rather separately to the localized eigenvector expansion at each point $\mathbf x$ of the domain). We find the latter approach to be somewhat more direct and efficient.

Accordingly, our derivations will be structured as follows. In Section \ref{sec:coeff_general} we discuss the general ideas of localization and high-frequency asymptotics. In Section \ref{sec:discont_loss} we give a direct derivation for the loss asymptotic in the case of the first (discontinuous) class $g$. In Section \ref{sec:gaussian_loss} we give a direct derivation for the loss asymptotic in the case of the second (Gaussian) class $g$. Then, in Section \ref{sec:q} we sketch the derivation of the coefficient asymptotic \eqref{tail asym} for both classes of $g$.

\subsection{General considerations}\label{sec:coeff_general}

Recall from Section \ref{sec:singular} and \ref{sec:eigvalues_asym_base} of the main text that the kernel $\Theta$ has a diagonal singularity:
\begin{equation}\label{eq:theta1}\Theta(\mathbf x,\mathbf x')=\theta_{\mathbf x}(\mathbf{x}-\mathbf{x}')+\ldots,\end{equation}
where $\theta_{\mathbf x}(\cdot)$ is a (possibly $\mathbf x$-dependent) even ($\theta_{\mathbf x}(\mathbf z)=\theta_{\mathbf x}(-\mathbf z)$) homogeneous function of degree $\alpha$:
\begin{equation}\label{eq:theta2}\theta_{\mathbf x}(c\mathbf z)=|c|^\alpha\theta_{\mathbf x}(\mathbf z).\end{equation}
It will be convenient to also consider the function $\psi_{\mathbf x}$ obtained by rescaling $\theta_{\mathbf x}$ by the coefficient $\mu(\mathbf x)$:
$$\psi_{\mathbf x}=\mu(\mathbf x)\theta_{\mathbf x}.$$
We denote by $\widetilde{\theta}_{\mathbf x}$ and $\widetilde{\psi}_{\mathbf x}$ the versions of the Fourier transforms of $\theta_{\mathbf x},\psi_{\mathbf x}$ defined as in Eq.~\eqref{eq:sing_Fourier} of the main text. Note that these functions are homogeneous with degree $-(d+\alpha).$

To analyze the evolution $g_t=e^{-t\ta}g$, we use the short-time Fourier transform (STFT) of $g$:  
\begin{align}
    F(\mathbf y, \mathbf k)={}&(2\pi)^{-d/2}\int_{\mathbb R^d} g(\mathbf x)\omega(\mathbf x-\mathbf y)e^{-i\mathbf k\cdot\mathbf x}d\mathbf x,\label{eq:fyk}\\
    g(\mathbf x)={}&(2\pi)^{-d/2}\int_{\mathbb R^d}\int_{\mathbb R^d} F(\mathbf y,\mathbf k)\omega(\mathbf x-\mathbf y)e^{i\mathbf k\cdot\mathbf x}d\mathbf yd\mathbf k.\label{eq:stft2}
\end{align}
Here $\omega$ is an even real smooth and compactly supported function such that $\int\omega^2=1$.
Roughly speaking, the coefficient $F(\mathbf y, \mathbf k)$ describes the component of $g$ having wave number $\mathbf k$ and localized at the point $\mathbf y.$

The stationary phase method (or its variant described in \cite{Birman_1970}) shows that if $f$ is a fixed function, then the leading term in the asymptotic of the action of the operator $\ta$ on the high-frequency function $f(\mathbf x)e^{i\mathbf k\cdot\mathbf x}$ (with $\mathbf k\to\infty$) can be written in terms of the Fourier transform of the diagonal singularity of the kernel:
\begin{align*}\int_{\mathbf R^d}\mu^{1/2}(\mathbf x)\Theta(\mathbf x,\mathbf x')\mu^{1/2}(\mathbf x')f(\mathbf x')e^{i\mathbf k\cdot\mathbf x'}d\mathbf x'\sim \widetilde{\psi}_{\mathbf x}(\mathbf k)f(\mathbf x)e^{i\mathbf k\cdot\mathbf x}. \end{align*}
This shows that for large $\mathbf k,$ we can think of the functions $f(\mathbf x)e^{i\mathbf k\cdot\mathbf x}$ as of approximate eigenvectors of the operator $\ta$ and accordingly also of the evolution $e^{-t\ta}.$ Then we can write
\begin{align}
    g_t(\mathbf x)={}& e^{-t\ta}g(\mathbf x)\nonumber\\
    ={}&  (2\pi)^{-d/2}\int_{\mathbb R^d}\int_{\mathbb R^d} F(\mathbf y,\mathbf k)e^{-t\ta} [\omega(\mathbf x-\mathbf y)e^{i\mathbf k\cdot\mathbf x}]d\mathbf yd\mathbf k\nonumber\\
    \stackrel{|\mathbf k|\gg 1}{\sim}{}& (2\pi)^{-d/2}\int_{\mathbb R^d}\int_{\mathbb R^d} F(\mathbf y,\mathbf k) e^{-t\widetilde{\psi}_{\mathbf x}(\mathbf k)}\omega(\mathbf x-\mathbf y)e^{i\mathbf k\cdot\mathbf x}d\mathbf yd\mathbf k.\label{eq:gt2}
\end{align}
To justify the assumption $|\mathbf k|\gg 1,$ observe that the function $e^{-t\widetilde{\psi}_{\mathbf x}(|\mathbf k|)}$ is close to 0 for $t|\mathbf k|^{-(d+\alpha)}\gg 1$ (i.e. for $|\mathbf k|\ll t^{\frac{1}{d+\alpha}}$), and is close to 1 for $t|\mathbf k|^{-(d+\alpha)}\ll 1$ (i.e. for $|\mathbf k|\gg t^{\frac{1}{d+\alpha}}$), so that at large $t$ the integral over $\mathbf k$ is indeed determined by the large-$|\mathbf k|$ asymptotic of $F(\mathbf y,\mathbf k).$

Next, we consider separately the two classes of functions $g$.

\subsection{Scenario 1: a discontinuous $g$}\label{sec:discont_loss} 
Suppose that $g$ is supported and smooth on the domain $\Omega,$ and has a discontinuity $\Delta g$ at the boundary $\partial \Omega.$ Consider the coefficient $F(\mathbf y,\mathbf k)$ defined in Eq.~\eqref{eq:fyk}.
If $\mathbf y$ is such that the support of $\omega(\cdot-\mathbf y)$ does not intersect the boundary $\partial \Omega$, then $g(\mathbf x')\omega(\mathbf x'-\mathbf y)$ is a smooth function of $\mathbf x',$ and the coefficient $F(\mathbf y,\mathbf k)$ will fall off faster than any power of $|\mathbf k|$ as $|\mathbf k|\to \infty,$ so the contribution of such $\mathbf y$ to the expansion \eqref{eq:gt2} will be negligible at large $t$. Assuming that the support of $\omega$ is small, it means that only $\mathbf y$ belonging to a narrow neighborhood of the boundary $\partial \Omega$ will contribute to \eqref{eq:gt2}. Accordingly, the function $g_t(\mathbf x)$ will also fall off quickly away from the boundary.

Suppose now that $\mathbf y$ lies near the boundary and the support of $\omega(\cdot-\mathbf y)$ intersects $\partial \Omega$.  It is convenient to consider first the one-dimensional case $d=1$. 

\subparagraph{Case $d=1$.} In this case the large-$k$ asymptotic of the coefficients $F(y,k)$ will be determined by the discontinuity of $g(x')\omega(x'-y)$ at the boundary point $x'=x_0\in\partial \Omega$:
$$F(y,k)\sim (2\pi)^{-1/2}\frac{\Delta g(x_0)\omega(x_0-y)}{ik}e^{-kx_0}.$$
Now we substitute this into Eq.~\eqref{eq:gt2}:
$$g_t(x)=(2\pi)^{-1}\Delta g(x_0)\int_{\mathbb R}\int_{\mathbb R} \frac{e^{-t\widetilde{\psi}_{\mathbf x}(k)}}{ik}\omega(x_0-y)\omega( x-y)e^{ik(x-x_0)}dydk,$$
where $x_0$ is the point of $\partial\Omega$ near which the point $x$ is located\footnote{For $d=1$, we take $\Omega$ to be a finite union of intervals, so $\partial\Omega$ consists of finitely many points.}. Regarding the function $\omega(x-y)$, one observes by rescaling the variable $k$ that at large $t$,  only a small (size-$t^{-\frac{1}{d+\alpha}}$) neighborhood of the boundary points  will contribute to the integral, so we can write $\omega(x-y)\sim \omega(x_0-y)$ and integrate out $y$ using the formula $\int_{\mathbb R}\omega^2=1:$
\begin{align}
g_t(x)\sim{}&(2\pi)^{-1}{\Delta g(x_0)}\int_{\mathbb R}\int_{\mathbb R} \frac{e^{-t\widetilde{\psi}_{\mathbf x}(k)}}{ik}\omega^2(x_0-y)e^{ik(x-x_0)}dydk\nonumber\\
={}&(2\pi)^{-1}{\Delta g(x_0)}\int_{\mathbb R} \frac{e^{-t\widetilde{\psi}_{\mathbf x}(k)}}{ik}e^{ik(x-x_0)}dk.\label{eq:gtxik}
\end{align}
We now want to estimate the full squared norm $\|g_t\|^2=\int_{\mathbb R}|g_t(x)|^2dx.$ On the whole $\mathbb R$, the function $g_t$ is given by the sum of contributions from different points $x_0\in\partial\Omega$:
\begin{align}\label{eq:gtxsum}
g_t(x)\sim{}&\sum_{x_0\in\partial\Omega}(2\pi)^{-1}{\Delta g(x_0)}\int_{\mathbb R} \frac{e^{-t\widetilde{\psi}_{\mathbf x}(k)}}{ik}e^{ik(x-x_0)}dk.
\end{align}
Observe that the functions $u_{x_0}(k)=\frac{e^{-t\widetilde{\psi}_{\mathbf x}(k)}}{ik}e^{-ikx_0}$ with different $x_0\in\partial\Omega$ become orthogonal in the limit $t\to\infty$ (since the Fourier transforms of these functions are localized at size-$t^{-\frac{1}{d+\alpha}}$ neighborhoods of the respective boundary points $x_0$).
Then, by the unitarity of Fourier transform,
\begin{align*}
\|g_t\|^2\sim{}& \sum_{x_0\in\partial\Omega}(2\pi)^{-1}|\Delta g(x_0)|^2\int_{\mathbb R}\frac{e^{-2t\widetilde{\psi}_{\mathbf x}(k)}}{|k|^2}dk\\
={}&\sum_{x_0\in\partial\Omega}|\Delta g(x_0)|^2\frac{1}{\pi}\int_0^\infty \frac{\exp(-2\widetilde{\psi}_{\mathbf x}(1)tk^{-(d+\alpha)})}{|k|^2}dk.
\end{align*}
The $t\to\infty$ asymptotic of the integral has been determined in Section \ref{sec:genlossformula} (see also formula \eqref{eq:loss_asymp} in the main text): setting $\nu=d+\alpha, \kappa=1,\Lambda=\widetilde{\psi}_{\mathbf x}(1), K=1$, we get 
\begin{align*}
\|g_t\|^2\sim{}& \sum_{x_0\in\partial\Omega}|\Delta g(x_0)|^2\frac{1}{\pi}K\Gamma\left(\frac{\kappa}{\nu} + 1\right) \left( 2 \Lambda  t\right)^{-\frac{\kappa}{\nu}}\\
={}&\sum_{x_0\in\partial\Omega}|\Delta g(x_0)|^2\frac{1}{\pi}\Gamma\Big(\frac{1}{d+\alpha} + 1\Big)(2\widetilde{\psi}_{\mathbf x}(1)t)^{-\frac{1}{d+\alpha}}.
\end{align*}

\subparagraph{Case $d>1$.} Assuming that the support of $\omega$ is small enough, we can approximate the boundary $\partial\Omega$ locally, in the support of $\omega(\cdot-\mathbf y)$, by a linear hyperplane $\{\mathbf x: \mathbf n\cdot\mathbf x =x_0\}$, where $\mathbf n$ is the inward unit normal to $\partial\Omega$, so that the function $g(\cdot)\omega(\cdot-\mathbf y)$ can be represented as a product of the Heaviside step function in the direction $\mathbf n$ and the smooth function $\omega(\cdot-\mathbf y)$:
$$g(\mathbf x')\omega(\mathbf x'-\mathbf y)\sim \mathbf 1_{\mathbf n\cdot\mathbf x'\ge x_0}(\mathbf x')\omega(\mathbf x'-\mathbf y).$$
The coefficient $F(\mathbf y,\mathbf k)$ is the Fourier transform of this function w.r.t. $\mathbf x'$. The Fourier transform of $\mathbf 1_{\mathbf n\cdot\mathbf x'\ge x_0}$ is a distribution concentrated on the line $l_{\mathbf n}=\{\mathbf k: \mathbf k=u\mathbf n, u\in\mathbb R\}.$ The Fourier transform of $\mathbf 1_{\mathbf n\cdot\mathbf x'\ge x_0}(\mathbf x')\omega(\mathbf x'-\mathbf y)$ w.r.t. $\mathbf x'$ is the convolution of this line distribution with the Fourier transform of $\omega(\mathbf x'-\mathbf y).$ Then for given $\mathbf y,$ by the smoothness of $\omega$,  the coefficients $F(\mathbf y,\mathbf k)$ are concentrated in a neighborhood of the line $l_{\mathbf n}$ in the $\mathbf k$-space. Let $\mathbf k_{\parallel}$ denote the projection of vector $\mathbf k$ to this line. Since $\widetilde{\psi}_{\mathbf x}(\mathbf k)$ is a homogeneous function, for large $t$ we can write $\widetilde{\psi}_{\mathbf x}(\mathbf k)\approx \widetilde{\psi}_{\mathbf x}(\mathbf k_{\parallel})$ in the integral \eqref{eq:gt2}:
\begin{align*}
    g_t(\mathbf x)\sim  (2\pi)^{-d/2}\int_{\mathbb R^d}\int_{\mathbb R^d} F(\mathbf y,\mathbf k) e^{-t\widetilde{\psi}_{\mathbf x}(\mathbf k_{\parallel})}\omega(\mathbf x-\mathbf y)e^{i\mathbf k\cdot\mathbf x}d\mathbf yd\mathbf k.
\end{align*}
We can now decompose the wave number $\mathbf k$ and the vector $\mathbf x$ into components parallel and orthogonal to the normal $\mathbf n:$
$$\mathbf k=\mathbf k_{\parallel}+\mathbf k_{\bot}=k_{\parallel}\mathbf n+\mathbf k_{\bot},\quad \mathbf x=\mathbf x_{\parallel}+\mathbf x_{\bot}= x_{\parallel}\mathbf n+\mathbf x_{\bot},$$
and perform integration over the component $\mathbf k_\bot$ in the above formula after substituting the expression for $F(\mathbf y,\mathbf k)$:
\begin{align*}
    g_t(\mathbf x)\sim{}&  (2\pi)^{-d}\int_{\mathbb R^d}\int_{\mathbb R^d}\int_{\mathbb R^d}  g(\mathbf x')\mathbf 1_{\mathbf n\cdot\mathbf x'\ge x_0}(\mathbf x')\omega(\mathbf x'-\mathbf y) e^{-t\widetilde{\psi}_{\mathbf x}(\mathbf k_{\parallel})}\omega(\mathbf x-\mathbf y)e^{i\mathbf k\cdot(\mathbf x-\mathbf x')}d\mathbf yd\mathbf kd\mathbf x'\\
    ={}&(2\pi)^{-d+1}\int_{\mathbb R^d}\int_{\mathbb R}\int_{\mathbb R}g(\widetilde{\mathbf x}')\mathbf 1_{x_{\parallel}'\ge x_0}(x_\parallel')\omega(\widetilde{\mathbf x}'-\mathbf y)e^{-t\widetilde{\psi}_{\mathbf x}(\mathbf k_{\parallel})}\omega(\mathbf x-\mathbf y)e^{ik_{\parallel}(x_{\parallel}- x_{\parallel}')}d\mathbf yd k_{\parallel}d x_{\parallel}',
\end{align*}
where $\widetilde{\mathbf x}'=\mathbf x_{\bot}+x_{\parallel}'\mathbf n$. We can now proceed similarly to the previous case $d=1$. Specifically, let $\mathbf x_0=\mathbf x_{\bot}+x_0\mathbf n$ be the projection of the point $\mathbf x'$ to the surface $\partial \Omega$. At large $t$ and $k_{\parallel}$, we can approximate $g(\widetilde{\mathbf x}')\approx g({\mathbf x}_0),\widetilde{\mathbf x}'\approx \mathbf x_0, \mathbf x\approx \mathbf x_0,$ and integrate out $\mathbf y$ using $\int\omega^2(\mathbf x_0-\mathbf y)d\mathbf y=1:$
\begin{align*}
    g_t(\mathbf x)\sim{}&  (2\pi)^{-d+1}\Delta g(\mathbf x_0)\int_{\mathbb R}\int_{\mathbb R}  \mathbf 1_{x_{\parallel}'\ge x_0}(x_\parallel')e^{-t\widetilde{\psi}_{\mathbf x}(\mathbf k_{\parallel})}e^{ik_{\parallel}(x_{\parallel}- x_{\parallel}')}d k_{\parallel}d x_{\parallel}'\\
    \sim{}&(2\pi)^{-d}\Delta g(\mathbf x_0)\int_{\mathbb R}\frac{e^{-t\widetilde{\psi}_{\mathbf x}(\mathbf k_{\parallel})}}{ik_\parallel}e^{ik_\parallel (x_\parallel-x_0)}dk_\parallel.
\end{align*}
This expression is analogous to the expression \eqref{eq:gtxik}, and similarly to the case $d=1$, we can now use it to obtain the asymptotic of $\|g_t\|^2$. Recall that in the case $d=1$, for each $x$ near the boundary we considered its projection $x_0$ to the boundary, and obtained the full integral $\|g_t\|^2$ by summing the contributions from different points $x_0$ (see Eq.~\eqref{eq:gtxsum} and subsequent formulas). In the present case $d>1$, we replace this summation by integration over the boundary $\partial\Omega$. By repeating the same steps as before, we then get
\begin{align*}
\|g_t\|^2\sim{}& \frac{1}{\pi}\Gamma\left(\frac{1}{d+\alpha} + 1\right)\int_{\partial\Omega}|\Delta g(\mathbf x)|^2 (2\widetilde{\psi}_{\mathbf x}(\mathbf n)t)^{-\frac{1}{d+\alpha}}dS\\
={}&\frac{1}{\pi}\Gamma\left(\frac{1}{d+\alpha} + 1\right)\int_{\partial\Omega}|\Delta g(\mathbf x)|^2 (\mu(\mathbf x)\widetilde{\theta}_{\mathbf x}(\mathbf n))^{-\frac{1}{d+\alpha}}dS \cdot (2t)^{-\frac{1}{d+\alpha}},
\end{align*}
yielding the loss asymptotic \eqref{eq:ltdisc}.

\subsection{Scenario 2: $g$ generated by a Gaussian process}\label{sec:gaussian_loss}
Suppose now that $g$ is generated by a Gaussian process with covariance $\Sigma(\mathbf x,\mathbf x')=\langle g(\mathbf x)g(\mathbf x')\rangle$, and that $\Sigma$ has a homogeneous singularity of degree $\beta$ on the diagonal $\mathbf x=\mathbf x':$
$$\Sigma(\mathbf x,\mathbf x')=\zeta_{\mathbf x}(\mathbf x'-\mathbf x)+\ldots,$$
where the dots denote terms of a higher smoothness, and $\zeta_{\mathbf x}$ is an $\mathbf x$-dependent even homogeneous function of degree $\beta:$
$$\zeta_{\mathbf x}(c\mathbf z)=|c|^\beta \zeta_{\mathbf x}(\mathbf z).$$
Similarly to the previously considered homogeneous functions, we denote by $\widetilde{\zeta}_{\mathbf x'}$ the Fourier transform of ${\zeta}_{\mathbf x'}$ defined using Eq.~(26) of the main text. 

To analyze the asymptotic of $\|g_t\|^2,$ we use again the representation \eqref{eq:gt2} in which we substitute the expansion for $F(\mathbf y,\mathbf k)$: 
\begin{align}
    g_t(\mathbf x)\sim  (2\pi)^{-d}\int_{\mathbb R^d}\int_{\mathbb R^d}\int_{\mathbb R^d} g(\mathbf x')\omega(\mathbf x'-\mathbf y) e^{-t\widetilde{\psi}_{\mathbf x}(\mathbf k)}\omega(\mathbf x-\mathbf y)e^{i\mathbf k\cdot(\mathbf x-\mathbf x')}d\mathbf yd\mathbf kd\mathbf x'.
\end{align}

Using as before the argument with rescaling, we see that the leading contribution to this integral comes at large $\mathbf k$ and small $\mathbf x-\mathbf x'.$ In particular, we can  write $\omega(\mathbf x'-\mathbf y)\approx\omega(\mathbf x-\mathbf y)$ and integrate $\mathbf y$ out:
\begin{align}g_t(\mathbf x)\sim  (2\pi)^{-d}\int_{\mathbb R^d}\int_{\mathbb R^d} g(\mathbf x') e^{-t\widetilde{\psi}_{\mathbf x'}(\mathbf k)}e^{i\mathbf k\cdot(\mathbf x-\mathbf x')}d\mathbf kd\mathbf x'.
\end{align}
Also, for further convenience, we have replaced $\mathbf x$ by $\mathbf x'$ in $\widetilde{\psi}_{\mathbf x}(\mathbf k).$

We now approximate $\|g_t\|^2$ by its expectation over the target functions $g$ generated by the Gaussian process\footnote{In general (if the Gaussian process does not have a small correlation length), a sampled value of $\|g_t\|^2$ need not be close to the expectation $\langle\|g_t\|^2\rangle$. However, one can show using the Wick-Isserlis formula that the variance $\langle(\|g_t\|^2-\langle\|g_t\|^2\rangle)^2\rangle$ scales with $t$ as $t^{-(d+2\beta)/(d+\alpha)}$, i.e. becomes asymptotically negligible compared to $\langle\|g_t\|^2\rangle^2$, which scales as $t^{-2\beta/(d+\alpha)}$. We plan to return to this point in a subsequent publication.}:
\begin{align}
    \|g_t\|^2\approx{}& \langle\|g_t\|^2\rangle\\
    ={}&\int_{\mathbb R^d}\langle g_t^2(\mathbf x)\rangle d\mathbf x\\
    \sim {}& (2\pi)^{-2d}\int_{\mathbb R^{5d}}\langle g(\mathbf x')g(\widetilde{\mathbf x}')\rangle
    e^{-t\widetilde{\psi}_{\mathbf x'}(\mathbf k)}e^{i\mathbf k\cdot(\mathbf x-\mathbf x')}
    e^{-t\widetilde{\psi}_{\widetilde{\mathbf x}'}(\widetilde{\mathbf k})}e^{-i\widetilde{\mathbf k}\cdot(\mathbf x-\widetilde{\mathbf x}')}
    d\mathbf kd\mathbf x'd\widetilde{\mathbf k}d\widetilde{\mathbf x}'d\mathbf x.
\end{align}
We integrate out $\mathbf x$ using the identity $\int_{\mathbb R^d} e^{i(\mathbf k-\widetilde{\mathbf k})\cdot\mathbf x}d\mathbf x=(2\pi)^d\delta(\mathbf k-\widetilde{\mathbf k}):$
\begin{align}
    \|g_t\|^2
    \sim {}& (2\pi)^{-d}\int_{\mathbb R^{3d}}\Sigma(\mathbf x',\widetilde{\mathbf x}')
    e^{-t\widetilde{\psi}_{\mathbf x'}(\mathbf k)-t\widetilde{\psi}_{\widetilde{\mathbf x}'}({\mathbf k})} e^{i\mathbf k\cdot(\widetilde{\mathbf x}'-\mathbf x')}
    d\mathbf kd\mathbf x'd\widetilde{\mathbf x}'.
\end{align}
We isolate now the singularity and apply the stationary phase method, obtaining the high-frequency approximation 
$$\int _{\mathbb R^d}\Sigma(\mathbf x',\widetilde{\mathbf x}')e^{-t\widetilde{\psi}_{\mathbf x'}(\mathbf k)-t\widetilde{\psi}_{\widetilde{\mathbf x}'}({\mathbf k})}e^{i\mathbf k\cdot(\widetilde{\mathbf x}'-\mathbf x')} d\widetilde{\mathbf x}'\sim \widetilde{\zeta}_{\mathbf x'}(\mathbf k)e^{-2t\widetilde{\psi}_{\mathbf x'}(\mathbf k)},\quad |\mathbf k|\gg 1.$$
This leads to
\begin{align}
    \|g_t\|^2
    \sim {}& (2\pi)^{-d}\int_{\mathbb R^{2d}}\widetilde{\zeta}_{\mathbf x'}(\mathbf k)e^{-2t\widetilde{\psi}_{\mathbf x'}(\mathbf k)}
    d\mathbf kd\mathbf x'.
\end{align}
To analyze the asymptotic of this integral at large $t$, we represent $\mathbf k$ as $|\mathbf k|\mathbf n,$ where $\mathbf n$ is a unit vector. Then, using the large-$\mathbf k$ asymptotics $\widetilde{\zeta}_{\mathbf x'}(\mathbf k)=\widetilde{\zeta}_{\mathbf x'}(\mathbf n)|\mathbf k|^{-(d+\beta)}$ and $\widetilde{\psi}_{\mathbf x'}(\mathbf k)=\widetilde{\psi}_{\mathbf x'}(\mathbf n)|\mathbf k|^{-(d+\alpha)},$
\begin{align*}
    \|g_t\|^2
    \sim {}& (2\pi)^{-d}\int_{\mathbb R^{d}}\int_{|\mathbf n|=1}\int_{0}^\infty \widetilde{\zeta}_{\mathbf x'}(\mathbf n)r^{-(d+\beta)}e^{-2t\widetilde{\psi}_{\mathbf x'}(\mathbf n)r^{-(d+\alpha)}}
    r^{d-1}d\mathbf x'dS dr\\
    ={}& (2\pi)^{-d}\int_{\mathbb R^{d}}\int_{|\mathbf n|=1} \widetilde{\zeta}_{\mathbf x'}(\mathbf n)\int_{0}^\infty r^{-(1+\beta)}e^{-2t\widetilde{\psi}_{\mathbf x'}(\mathbf n)r^{-(d+\alpha)}}
    d\mathbf x'dS dr\\
    \sim{}&\frac{1}{(2\pi)^d\beta}\Gamma\Big(\frac{\beta}{d+\alpha}+1\Big)\int_{\mathbb R^{d}}\int_{|\mathbf n|=1}\widetilde{\zeta}_{\mathbf x'}(\mathbf n)\widetilde{\psi}^{-\frac{\beta}{d+\alpha}}_{\mathbf x'}(\mathbf n)d\mathbf x'dS\cdot (2t)^{-\frac{\beta}{d+\alpha}}\\
    ={}&\frac{1}{(2\pi)^d\beta}\Gamma\Big(\frac{\beta}{d+\alpha}+1\Big)\int_{\mathbb R^{d}}\int_{|\mathbf n|=1}\widetilde{\zeta}_{\mathbf x}(\mathbf n)(\mu(\mathbf x)\widetilde{\theta}_{\mathbf x}(\mathbf n))^{-\frac{\beta}{d+\alpha}}d\mathbf xdS\cdot  (2t)^{-\frac{\beta}{d+\alpha}},
\end{align*}
which yields the loss asymptotic \eqref{eq:ltgauss}.

In the case when both GP and operator $\mathcal{A}$ originate from the same shallow ReLU network, Fourier transforms of diagonal singularities have similar angular dependence and integration over sphere can be performed analytically.  We know from section \ref{sec:gammax} that if the kernel singularities are based on the angle $\varphi(\mathbf{x},\mathbf{x'})$ between the input points $\mathbf{x},\mathbf{x}'$, then    $\widetilde{\zeta}_\mathbf{x}(\mathbf{z}) \propto \varphi^\beta(\mathbf{x}, \mathbf{x}+\mathbf{z})$ and $\widetilde{\theta}_\mathbf{x}(\mathbf{z}) \propto \varphi^\alpha(\mathbf{x}, \mathbf{x}+\mathbf{z})$. Then the respective Fourier transforms admit the forms $\widetilde{\zeta}(\mathbf{n}) = P(\mathbf{x}) |\mathbf{n}'|^{-d-\beta}$ and $\widetilde{\theta}(\mathbf{n}) = Q(\mathbf{x})  |\mathbf{n}'|^{-d-\alpha}$, with $\mathbf{n}'$ the same as $\mathbf{n}$ except the first dimension: $n'_1=\tfrac{r(\mathbf{x})}{\sigma_b} n_1$. We write the sphere integral as
\begin{equation}\label{eq:sphere_int}
    \begin{split}
        &\int_{|\mathbf n|=1} dS |\mathbf{n}'|^{-d-\beta}\left(|\mathbf{n}'|^{-d-\alpha}\right)^{-\tfrac{\beta}{d+\alpha}} = \int_{|\mathbf n|=1} dS |\mathbf{n}'|^{-d} = \int_{|\mathbf n|=1} dS \left(\frac{r^2}{\sigma_b^2}n_1^2 + (n_2^2+\ldots+n_d^2)\right)^{-\tfrac{d}{2}} \\
        & \stackrel{(1)}{=} \int_0^\pi d\rho (\sin \rho)^{d-2} \int_{|\mathbf{\widetilde{n}}|=1} d\widetilde{S} \left(\frac{r^2}{\sigma_b^2} \cos^2(\rho) + \sin^2(\rho)\right)^{-\tfrac{d}{2}} = S_{d-2} \int_0^\pi d\rho (\sin \rho)^{d-2} \left(\frac{r^2}{\sigma_b^2} \cos^2(\rho) + \sin^2(\rho)\right)^{-\tfrac{d}{2}} \\
        & = S_{d-2} \int_{-\infty}^{+\infty} d (\cot \rho) \left( \frac{r^2}{\sigma_b^2}\cot^2(\rho) + 1\right)^{-\tfrac{d}{2}} = \frac{\sigma_b}{r} S_{d-2} \int_{-\infty}^{+\infty} dz (z^2+1)^{-\tfrac{d}{2}} \stackrel{(2)}{=} \frac{\sigma_b}{r} S_{d-1}
    \end{split}
\end{equation}
Here in (1) we split integration over sphere $|\mathbf{n}|=1$ over the first axis and remaining $d-2$ dimensional sphere $|\mathbf{\widetilde{n}}|=1$: $n_1=\cos\rho$ and $(n_2,\ldots,n_d)=\mathbf{\widetilde{n}} \sin\rho$. Finally in (2) the value of the integral over $z$ equals $S_{d-1}/S_{d-2}$, which can be inferred from the spherically symmetric case $r/\sigma_b=1$. 

\subsection{The coefficient distributions }\label{sec:q}
The derivations given above bypass the explicit computation of the cumulative distribution function $s_n$ for the coefficients $c_n$ of the expansion of $g$ w.r.t. the eigenbasis of the operator $\ta$ (see Eqs.~\eqref{eq:s_partial_sum},\eqref{tail asym}). These can be derived (at least heuristically) using essentially the same approach based on localized approximate eigendecomposition, but this time accompanied by the count of the total contribution of the coefficients corresponding to the given eigenvalue threshold from all the points of the domain.

It is convenient to introduce the partial sum $Q(\lambda)$ of the coefficients $|c_n|^2$ defined as in \eqref{eq:s_partial_sum} but expressed in terms of the eigenvalue threshold $\lambda$:
\begin{equation}\label{eq:ql0}Q(\lambda)=\sum_{n:\lambda_n<\lambda}|c_n|^2.\end{equation}
The large-$n$ asymptotic of $s_n$ corresponds to the small-$\lambda$ asymptotic of $Q(\lambda)$.

For Scenario 1 (discontinuous $g$), the resulting expression is 
\begin{equation}\label{eq:qldisc}
Q(\lambda)\sim \frac{1}{\pi}\int_{\partial\Omega}|\Delta g(\mathbf x)|^2 (\mu(\mathbf x)\widetilde{\theta}_{\mathbf x}(\mathbf n))^{-\frac{1}{d+\alpha}}dS\cdot\lambda^{\tfrac{1}{d+\alpha}}.\end{equation}

For Scenario 2 (Gaussian $g$), the resulting expression is 
\begin{equation}\label{eq:qlgauss}
Q(\lambda)\sim \frac{1}{(2\pi)^d\beta}\int_{\mathbb R^{d}}\int_{|\mathbf n|=1}\widetilde{\zeta}_{\mathbf x}(\mathbf n)(\mu(\mathbf x)\widetilde{\theta}_{\mathbf x}(\mathbf n))^{-\frac{\beta}{d+\alpha}}d\mathbf xdS\cdot\lambda^{\tfrac{\beta}{d+\alpha}}.\end{equation}

Since we have already found the loss asymptotics for both scenarios (Eqs.~\eqref{eq:ltdisc},\eqref{eq:ltgauss}), we can establish the above expressions by showing that in either case $Q(\lambda)\sim a \lambda^b$ with some specific exponent $b$; the coefficient $a$ can then be deduced from the respective loss coefficient. 

To find the exponent $b$, we consider again the STFT representation \eqref{eq:fyk},\eqref{eq:stft2}. Suppose that the function $\omega$ lives on a small scale $\tfrac{1}{M}$:
$$\omega(\mathbf x)=M^{d/2}\omega_0(M\mathbf x),$$
and suppose that the domain is accordingly decomposed into $\propto M^d$ independent  ``$\mathbf y$--cells''. We can think of the respective STFT coefficients $F(\mathbf y,\mathbf k)$ as representing the actual coefficients in the eigenvector expansion. Consider now separately the two scenarios.

\paragraph{Scenario 1: a discontinuous $g$.}
The coefficients $F(\mathbf y,\mathbf k)$ are negligible for cells not intersecting the boundary $\partial \Omega$. Suppose that the cell intersects $\partial\Omega$. Then the coefficients $F(\mathbf y,\mathbf k)$ in this cell vanish outside the line $\mathbf k=u\mathbf n$ in the $\mathbf k$-space, where $\mathbf n$ is the unit normal to $\partial\Omega$. For $\mathbf k=u\mathbf n,$ we have 
\begin{equation*}
|F(\mathbf y, \mathbf k)|\propto M^{d/2}M^{1-d}\tfrac{|\Delta g(\mathbf y)|}{|\mathbf k|}=M^{1-d/2}\tfrac{|\Delta g(\mathbf y)|}{|\mathbf k|}.
\end{equation*}
We assume now that for suitable discrete wave numbers $\mathbf k$ the coefficients $F(\mathbf y,\mathbf k)$ are associated to respective approximate eigenvectors of $\widetilde{\mathcal A}$ in the $\mathbf y$-cell, and  estimate the respective contribution of the coefficients to the sum $S(\lambda)$. The discreteness results from the finite size of the support of $\omega$: the density of the eigenvalues scales with $M$ as $M^{-d}$. The respective discrete constants $u$ for the relation $\mathbf k=u\mathbf n$ scale as integer multiples of $M$, i.e. $u_l\sim lM$. Accordingly, the   contribution of the coefficients $F(\mathbf y,\mathbf k)$ in the $\mathbf y$-cell to $Q(\lambda)$ can be estimated as 
\begin{align}\label{eq:sumfyk}
\sum_{\mathbf k:\mathbf k=u_l\mathbf n,\lambda_{\mathbf k}<\lambda}|F(\mathbf y,\mathbf k)|^2\sim{}& M^{-1}\int_{k_0}^\infty M^{2-d}\tfrac{|\Delta g|^2}{k^2}dk\nonumber\\
\sim{}& M^{1-d}\tfrac{|\Delta g|^2}{k_0},
\end{align}
where the wave number $k_0=|u\mathbf n|$ corresponds to the eigenvalue $\lambda$:
$$\widetilde\theta_{\mathbf x}(u\mathbf n)=\lambda.$$
We can find $k_0$ using the homogeneity of $\widetilde\theta_\mathbf x$: since $\widetilde\theta_{\mathbf x}(u\mathbf n)=|u|^{-d-\alpha}\widetilde\theta_{\mathbf x}(\mathbf n),$ we have
$$k_0=|u|=(\tfrac{\widetilde\theta_{\mathbf x}(\mathbf n)}{\lambda})^{1/(d+\alpha)}.$$
Substituting this into Eq.~\eqref{eq:sumfyk} and taking into account that there are $\propto M^{d-1}$ cells intersecting $\partial\Omega$, we find Eq.~\eqref{eq:qldisc} up to a coefficient.

\paragraph{Scenario 2: $g$ generated by a Gaussian process.}
As before, finding the exponent $b$  can be reduced to estimating the asymptotic of $|F(\mathbf y,\mathbf k)|^2$ at a fixed $\mathbf y$ and large $\mathbf k$. Computing the expectation, we get
\begin{align*}\langle |F(\mathbf y,\mathbf k)|^2\rangle={}&(2\pi)^{-d}\int_{\mathbb R^d}\int_{\mathbb R^d} \omega(\mathbf x-\mathbf y)\omega(\mathbf x'-\mathbf y)\\
&\times\Sigma(\mathbf x,\mathbf x')e^{-i\mathbf k\cdot(\mathbf x-\mathbf x')}d\mathbf x d\mathbf x'\\
\propto{}& |\mathbf k|^{-(d+\beta)}.
\end{align*}
Since the wave vector $\mathbf k$ corresponds to an eigenvalue $\lambda\propto |\mathbf k|^{-(d+\alpha)}$ and since the density of the wave numbers $\mathbf k\in\mathbb R^d$ associated with localized eigenvectors of $\widetilde{\mathcal A}$ scales as $M^{-d}$, we can write
\begin{align}
\sum_{\mathbf k:\lambda_{\mathbf k}<\lambda}|F(\mathbf y,\mathbf k)|^2\propto{}& M^{-d}\int_{|\mathbf k|<\lambda^{-1/(d+\alpha)}} |\mathbf k|^{-(d+\beta)}d\mathbf k\nonumber\\
\propto{}& M^{-d}\lambda^{\beta/(d+\alpha)}.\nonumber
\end{align}
Collecting the contributions to $Q(\lambda)$ from all $\sim M^d$ $\mathbf y$-cells, we thus get $$Q(\lambda)\sim a\lambda^{\beta/(d+\alpha)},$$ as claimed.

By expressing $\lambda$ through $n$ with the help of the eigenvalue asymptotic \eqref{lambda asym} and Eq.~\eqref{eq:nu1ad}, we can also cast the obtained formulas for $Q(\lambda)$ in the form $s_n\sim Kn^{-\kappa}$ as in Eq.~\eqref{tail asym}.

For Scenario 1 (discontinuous $g$), the resulting coefficient and exponent are 
\begin{align*}\kappa={}&\tfrac{1}{d},\\
K={}&\Big(\tfrac{1}{\pi}\int_{\partial \Omega}|\Delta g(\mathbf x)|^2 (\mu(\mathbf x)\widetilde{\theta}_{\mathbf x}(\mathbf n))^{-\frac{1}{d+\alpha}}dS\Big)\Big(\int \gamma_\mathbf x\mu^{\frac{d}{d+\alpha}}(\mathbf x)d\mathbf x\Big)^{1/d}.\end{align*}

For Scenario 2 (Gaussian $g$), the resulting coefficient and exponent are 
\begin{align*}\kappa={}&\tfrac{1}{d},\\
K={}&\Big(\tfrac{1}{(2\pi)^d\beta}\int_{\mathbb R^{d}}\int_{|\mathbf n|=1}\widetilde{\zeta}_{\mathbf x}(\mathbf n)(\mu(\mathbf x)\widetilde{\theta}_{\mathbf x}(\mathbf n))^{-\frac{\beta}{d+\alpha}}d\mathbf xdS\Big)\Big(\int \gamma_\mathbf x\mu^{\frac{d}{d+\alpha}}(\mathbf x)d\mathbf x\Big)^{1/d}.\end{align*}

\newpage
\section{Extensions}\label{sec:SM_extension}
In this section we derive results of section \ref{sec:extensions} in the paper.

\subsection{Activations of different smoothness}
We consider a shallow network in NTK regime with activation function $\phi_q(z)=(z)_+^{\;q}, \; q>0$. Output covariance $\Sigma_q$ and NTK $\Theta_q$ for such network can be written as 
\begin{align}
    &\Sigma_q(\mathbf{x},\mathbf{x}') = \sigma_w^2 \Big\langle \big(z(\mathbf{x})\big)_+^{\;q} \big(z(\mathbf{x}')\big)_+^{\;q} \Big\rangle \\
    &\label{eq:NTK_q}
    \Theta_q(\mathbf{x},\mathbf{x}') = \Sigma_q(\mathbf{x},\mathbf{x}') + \sigma_w^2 (\sigma_w^2 \mathbf{x}\cdot \mathbf{x'}+\sigma_b^2) q^2 \Big\langle \big(z(\mathbf{x})\big)_+^{\;q-1} \big(z(\mathbf{x}')\big)_+^{\;q-1} \Big\rangle
\end{align}
Here the average is taken w.r.t. pair of Gaussian random variables $z(\mathbf{x}),z(\mathbf{x}')$ with zero mean and covariance
\begin{equation}
    \big\langle (z(\mathbf{x}),z(\mathbf{x}')^T (z(\mathbf{x}),z(\mathbf{x}')\big\rangle = 
    \begin{pmatrix}
    r^2(\mathbf{x}) & r(\mathbf{x}) r(\mathbf{x}') \varphi(\mathbf{x}, \mathbf{x'})\\
    r(\mathbf{x}) r(\mathbf{x}') \varphi(\mathbf{x}, \mathbf{x'}) & r^2(\mathbf{x}')
    \end{pmatrix}
\end{equation}
Such averages were calculated in \cite{NIPS2009_kernel} for integer $q$, but we take intermediate integral representation (eqs. (3), (16)) from this paper, which we will analyze for general $q$. As usual, we omit explicit $\mathbf{x},\mathbf{x}'$ dependence for brevity.
\begin{equation}\label{eq:av_q}
    \big\langle (z)_+^{\;q} (z')_+^{\;q} \big\rangle = \frac{1}{2\pi} r^q r'^q \Gamma(q+1) \big(\sin\varphi\big)^{2q+1}\int\limits_0^{\tfrac{\pi}{2}} \frac{\big(\cos \psi\big)^q}{\big(1-\cos\varphi\cos\psi\big)^{q+1}} d\psi
\end{equation}
Let's denote  the integral in \eqref{eq:av_q} by $I_q(\varphi)$. We will transform it so it has the form of integral representation of the hypergeometric function $_2F_1$
\begin{equation}\label{eq:I_q}
\begin{split}
    I_q(\varphi) &= \int\limits_0^{\tfrac{\pi}{2}} \frac{\big(\cos \psi\big)^q}{\big(1-\cos\varphi\cos\psi\big)^{q+1}} d\psi \\
    &=\int\limits_0^1\frac{y^q}{\sqrt{1-y^2}(1-\cos\varphi y)^{q+1}}dy, \quad y=\cos\psi \\
    &=\frac{1}{(1-\cos\varphi)^{q+1}}\int\limits_0^1 t^q (1-t)^{-\tfrac{1}{2}}\left(1+\frac{1+\cos\varphi}{1-\cos\varphi}t\right)^{-\tfrac{1}{2}}dt, \quad t = \frac{y(1-\cos\varphi)}{1-\cos\varphi y}
\end{split}
\end{equation}
The hypergeometric function $_2F_1(a,b;c;z)$ has the following integral representation and asymptotic expansion at $z=-\infty$:
\begin{align}
    &\label{eq:2F1_int_rep}
    _2F_1(a,b;c;z) = \frac{\Gamma(c)}{\Gamma(b)\Gamma(c-b)} \int\limits_0^1 t^{b-1}(1-t)^{c-b-1}(1-tz)^{-a}dt \\
    &\label{eq:2F1_asym}
    _2F_1(a,b;c;-z) = z^{-b}\frac{\Gamma(a-b)\Gamma(c)}{\Gamma(a)\Gamma(-b+c)}\left(1+\sum_{n \geq 1}g_n z^{-n}\right) +z^{-a} \sum_{n \geq 0} f_n z^{-n} 
\end{align}
Here $g_n$ are $f_n$ are coefficients in the asymptotic expansion. Comparing our integral $I_q$ with integral representation of $_2F_1$ we see that it is indeed hypergeometric function with parameters $a=\tfrac{1}{2},b=q+1,c=q+\tfrac{3}{2}$ and argument $z= -\tfrac{1+\cos\varphi}{1-\cos\varphi}$. Singularity at $\varphi=0$ is located at hypergeometric function argument $z=-\infty$, therefore we need exactly asymptotic \eqref{eq:2F1_asym} to analyze singularity. Substituting our values of $_2F_1$ parameters we obtain the following asymptotic expansion at $\varphi=0$
\begin{equation}\label{eq:I_q_asym}
\begin{split}
    (\sin \varphi)^{2q+1} I_q(\varphi) &= \frac{\Gamma(q+1)\Gamma(-\tfrac{1}{2}-q)}{\sqrt{\pi}}\frac{(\sin \varphi)^{2q+1}}{(1+\cos\varphi)^{q+1}}\left[1+\sum_{n \geq 1}g_n \Big(\frac{1-\cos\varphi}{1+\cos\varphi}\Big)^{n}\right] \\
    &+ \frac{(\sin \varphi)^{2q+1}}{(1-\cos\varphi)^{q+\tfrac{1}{2}}(1+\cos\varphi)^{\tfrac{1}{2}}} \sum_{n \geq 0}f'_n \Big(\frac{1-\cos\varphi}{1+\cos\varphi}\Big)^{n}
\end{split}
\end{equation}
As it is written now, the asymptotic expansion above is not an expansion in powers $\varphi$, but it can be turned into one by replacing functions of $\varphi$ with their Taylor expansions. In particular, $\sin\varphi = \varphi+O(\varphi^3)$, $1-\cos\varphi=\tfrac{1}{2}\varphi^2 + O(\varphi^4)$ and $1+\cos\varphi = 2 + O(\varphi^2)$. In the asymptotic expansion \eqref{eq:I_q_asym} the second term starting from $\varphi^0$ is the leading one. However, it contains only even powers $\varphi^{2n}$, which are all regular. On the contrary, the first term starts with $\varphi^{2q+1}$ and it is singular for all $q$ except half-integers. Taking the leading singular term from \eqref{eq:I_q_asym} we obtain the leading singular term of NTK \eqref{eq:NTK_q}
\begin{equation}\label{eq:Theta_q_sing}
    \Theta_{q,\text{sing}} = \frac{\sigma_w^2}{2\pi} r^{2q} q^2 \frac{\Gamma^2(q)\Gamma(\tfrac{1}{2}-q)}{\sqrt{\pi}2^q} \varphi^{2q-1}
\end{equation}
Combining this with the $\gamma$ coefficient from \eqref{eq:gamma_alpha2} with $\alpha=2q-1$ and $A(\mathbf{x})$ deduced from \eqref{eq:Theta_q_sing} we get eigenvalue asymptotic coefficient $\Lambda_q$
\begin{equation}\label{eq:Lambda_q}
\Lambda_q ={} \sigma_w^{\alpha+2} \sigma_b^{\tfrac{\alpha}{d}} q^2 (2\pi)^{d+q-2} \frac{\Gamma\big(\tfrac{d+\alpha}{2}\big) \Gamma^2(\tfrac{\alpha+1}{2})}{\Big(\Gamma(\tfrac{d}{2}+1)\Big)^{\frac{d+\alpha}{d}}} \left\langle \mu(\mathbf{x})^{-\tfrac{\alpha}{d+\alpha}}r(\mathbf{x})^{\frac{2d-\alpha d -\alpha}{d+\alpha}}\right\rangle^{\frac{d+\alpha}{d}}_\mu
\end{equation}
In the case of half-integer $q$ the coefficient in \eqref{eq:Theta_q_sing} diverges due to gamma function $\Gamma(\tfrac{1}{2}-q)$ having simple poles at positive half integer $q$. Quite interestingly, the same delta function is found in $\gamma_{d,\alpha}^{(d+\alpha)/d}$ and they cancel. Therefore, the final constant $\Lambda_q$ formally has a meaningful limit at half integer $q$. However, existence of a limit does not prove that at half-integer $q$ eigenvalues have an asymptotic with constant \eqref{eq:Lambda_q}. The half integer case should be studied separately and we leave it for the future work.

\subsection{Deep networks}
We consider deep network $L>2$ in the NTK regime and with ReLU activation function. Covariances and NTK's of intermediate layers are calculated as

\begin{equation}\label{eq:recur}
    \begin{cases}
    & \Sigma^{(l+1)}(\mathbf{x},\mathbf{x}') = \sigma_w^2 \big\langle \phi(z^l(\mathbf{x}))\phi(z^{l}(\mathbf{x}')\big\rangle + \sigma_b^2\\
    &
    \Theta^{(l+1)}(\mathbf{x},\mathbf{x}') = \Sigma^{(l+1)}(\mathbf{x},\mathbf{x}') +\sigma_w^2 \Theta^{(l)}(\mathbf{x},\mathbf{x}')\big\langle \dot{\phi}(z^l(\mathbf{x}))\dot{\phi}(z^l(\mathbf{x}'))\big\rangle
    \end{cases}
\end{equation}

Here, as in the paper, $z^l(\mathbf{x})$ is a GP with covariance $\Sigma_l(\mathbf{x},\mathbf{x'})$. Parametrizing covariance as 
\begin{equation}\label{eq:covariance_param}
    \Sigma^{(l)}(\mathbf{x},\mathbf{x}') = 
    \begin{pmatrix}
    &r_l(\mathbf{x})^2 & r_l(\mathbf{x}) r_l(\mathbf{x}') \cos \varphi_l(\mathbf{x},\mathbf{x}')\\
    &r_l(\mathbf{x}) r_l(\mathbf{x})' \cos \varphi_l & r_l(\mathbf{x}')^2
    \end{pmatrix}
\end{equation}
From this point we again drop $\mathbf{x},\mathbf{x}'$ dependence. Using parametrization \eqref{eq:covariance_param} we rewrite recursive relations \eqref{eq:recur}
\begin{equation}\label{eq:recur_2}
        \begin{cases}
    & \Sigma^{(l+1)} = \frac{\sigma_w^2}{2\pi} r_l r_l' \big(\sin \varphi_l +\cos \varphi_l (\pi-\varphi_l)\big) + \sigma_b^2\\
    &
    \Theta^{(l+1)} = \Sigma^{(l+1)} + \Theta^{(l)}\frac{\sigma_w^2}{2\pi} (\pi-\varphi_l)
    \end{cases}
\end{equation}

We see that NTK's $\Theta^{(l)}$ can be fully expressed through $\varphi_l$ and $r_l$. From \eqref{eq:recur_2} the recursive relations for $\varphi_l, r_l$ are  

\begin{equation}\label{eq:recur_3}
        \begin{cases}
    & r_{l+1}^2 = \frac{\sigma_w^2}{2} r_l^2 + \sigma_b^2\\
    & \cos \varphi_{l+1} = \frac{1}{r_{l+1}r_{l+1}'} \Big[\frac{\sigma_w^2}{2\pi} r_l r_l' \big(\sin \varphi_l +\cos \varphi_l (\pi-\varphi_l)\big) + \sigma_b^2\Big]
    \end{cases}
\end{equation}

From these equations we see that $\varphi_{l+1} = 0$ only when $\varphi_l=0$ and $R_l=r_l'$. This, in turn, happens only when $\mathbf{x}=\mathbf{x}'$. Using starting values $r_1^2(\mathbf{x})=\sigma_w^2 |\mathbf{x}|^2 +\sigma_b^2$ and $\varphi(\mathbf{x},\mathbf{x}')$ defined in \eqref{eq:phi1_def}, and the fact that $\arccos(z)$ is smooth everywhere except $z=-1,1$, one can see that $r_l(\mathbf{x})$ is smooth everywhere and $\varphi(\mathbf{x},\mathbf{x}')$ is smooth everywhere except the diagonal $\mathbf{x}=\mathbf{x}'$. Therefore, NTKs $\Theta^{(L)}$ are smooth away from diagonal and might have a singularity on it. From \eqref{eq:recur_2} and smoothness of $r_l(\mathbf{x})$ we see that the only source of singularity are $\varphi_l$.

To find singular expression for $\varphi_l$ let's assume that $\varphi_l=O(|\mathbf{x}-\mathbf{x}'|)$ and carefully expand second equation in \eqref{eq:recur_3} up to second order in $|\mathbf{x}-\mathbf{x'}|$. To do this we note that $|r_l-r_l'|=O(|\mathbf{x}-\mathbf{x}'|)$ and $\sin \varphi_l +\cos \varphi_l (\pi-\varphi_l) = \pi(1-\varphi_l^2/2) + O(\varphi_l^3)$
\begin{equation}\label{eq:phi_l_exp}
    \varphi_{l+1}^2 = \frac{\sigma_w^2 r_l^2}{2 r_{l+1}^2}\Big(\varphi_l^2 + (r_l-r_l')^2\frac{\sigma_b^2}{r_{l+1}^2}\Big) + O(|\mathbf{x}-\mathbf{x}'|^3) 
\end{equation}
Thus we confirmed our assumption $\varphi_l=O(|\mathbf{x}-\mathbf{x}'|)$. In the leading order both $(r_l-r_l')^2$ and $\varphi_1^2$ are homogeneous functions of degree 2. Combining this with \eqref{eq:phi_l_exp} we see that all $\varphi_l$ are homogeneous functions of degree 1 in the leading order of $\mathbf{x}-\mathbf{x}'$. 

Now we find the leading singular part of $\Theta^{(l+1)}$. We will see that the leading singular part has degree 1, therefore we assume the following NTK expansion
\begin{equation}
    \Theta^{(l)} = \Theta^{(l)}_{\text{diag}} - \sum\limits_{m=1}^{l-1} a_m^{(l)} \varphi_m + O(|\mathbf{x}-\mathbf{x}|^2)
\end{equation}
Here $\Theta^{(l)}_{\text{diag}}$ is the value of NTK at the diagonal, $a_m^{(l)}$ are constants and the sum represents leading singular part of the NTK $\Theta^{(l)}_{\text{sing}}$. The recursion relation \eqref{eq:recur_2} can be now written as
\begin{equation}
    \Theta^{(l+1)}_{\text{diag}} + \Theta^{(l+1)}_{\text{sing}} + O(|\mathbf{x}-\mathbf{x}|^2)= r_{l+1}^2 + \big(\Theta^{(l)}_{\text{diag}} + \Theta^{(l)}_{\text{sing}}\big) \frac{\sigma_w^2}{2\pi} (\pi-\varphi_l) +O(|\mathbf{x}-\mathbf{x}|^2)
\end{equation}
From this we extract recursive relations for diagonal and singular parts of NTK
\begin{equation}\label{eq:NTK_diag_sing_recur}
    \begin{split}
    & \Theta^{(l+1)}_{\text{diag}} = r_{l+1}^2 + \frac{\sigma_w^2}{2} \Theta^{(l)}_{\text{diag}} \\
    & \Theta^{(l+1)}_{\text{sing}} = -\frac{1}{2\pi}\Theta^{(l)}_{\text{diag}} \varphi_l + \frac{\sigma_w^2}{2} \Theta^{(l)}_{\text{sing}}
    \end{split}
\end{equation}
Constants $a_m^{(l)}$ can be explicitly extracted from this relations. Since $\Theta_{\text{diag}}^{(l)}>0$ we can see that all $a_m^{(l)}>0$. It means that the leading singular terms of order $O(|\mathbf{x}-\mathbf{x}|)$ will not cancel each over, thus confirming that the leading singularity in NTK has homogeneity degree 1.

\subsection{MF regime}
The NTK of network in MF regime is given by 
\begin{equation}
    \label{eq:MF_NTK}
    \Theta(\mathbf{x},\mathbf{x'}) = \int  \nabla_{\tilde{\mathbf{w}}}\tilde{\phi}(\tilde{\mathbf{w}}, \mathbf{x}) p(\tilde{\mathbf{w}}) \nabla_{\tilde{\mathbf{w}}}\tilde{\phi}(\tilde{\mathbf{w}}, \mathbf{x}')d\tilde{\mathbf{w}}
\end{equation}

Here $\tilde{\phi}(\tilde{\mathbf{w}},\mathbf{x})=c \phi(\mathbf{w}\cdot \mathbf{x}+b)$ is a computation of a single neuron in shallow network and  $p(\tilde{\mathbf{w}})=p(c,\mathbf{w},b)$ is a distribution of a neuron parameters. In the case of ReLU activation $\phi(z)=(z)_+$ we rewrite it as 
\begin{equation}\label{eq:MF_NTK_relu}
\begin{split}
        \Theta(\mathbf{x},\mathbf{x'}) &= \int (\mathbf{w}\cdot\mathbf{x}+b)_+ (\mathbf{w}\cdot\mathbf{x}'+b)_+ dp_{0}(\mathbf{w},b) \\
        &+(1+\mathbf{x}\cdot\mathbf{x'}) \int H(\mathbf{w}\cdot\mathbf{x}+b) H(\mathbf{w}\cdot\mathbf{x}'+b) dp_{2}(\mathbf{w},b) \\
        &= I_1(\mathbf{x},\mathbf{x'}) + (1+\mathbf{x}\cdot\mathbf{x'})I_2(\mathbf{x},\mathbf{x'})
\end{split}
\end{equation}
Here the first integral $I_1$ corresponds to taking the gradient w.r.t. $c$ in \eqref{eq:MF_NTK} and the second integral $I_2$ corresponds to taking the gradients w.r.t. $\mathbf{w}$ and $b$; $p_{0}(\mathbf{w},b)$ and $p_{2}(\mathbf{w},b)$ are the $0$'th and $2$'nd moment of the distribution $p(c,\mathbf{w},b)$ w.r.t. the variable $c$; $H(z)$ is the Heaviside step function.

Now suppose that $p_{0}$ and $p_{2}$ are sufficiently smooth and fall off quickly at infinity. To analyze the smoothness of the NTK \eqref{eq:MF_NTK_relu} let's differentiate it using that $\tfrac{d}{dz}(z)_+ = H(z)$ and $\tfrac{d}{dz}H(z) = \delta(z)$. We start with the second integral $I_2$. The first derivative produces one delta function, and, together with the left Heaviside function, they are located on hyperplanes $\mathbf{w}\cdot\mathbf{x}+b=0$ and $\mathbf{w}\cdot\mathbf{x}'+b=0$ in the $(\mathbf{w},b)$ space. First we consider a neighborhood of points $\mathbf{x} \ne \mathbf{x}'$ - the corresponding integral continuously depend on $\mathbf{x},\mathbf{x}'$. If we differentiate the second time we will have two delta function, which restrict the integral to $d-1$ dimensional subspace of $(\mathbf{w},b)$ space, which is also continuously depend on $\mathbf{x},\mathbf{x}'$. Further derivatives can be translated to differentiating $p_2(\mathbf{w},b)$, with the result being continuous as long as $p_2$ is sufficiently smooth. Thus we established that the second integral is smooth at $\mathbf{x}\ne\mathbf{x'}$. Now let's turn to the diagonal $\mathbf{x}=\mathbf{x'}$, where two hyperplanes coincide, and consider the first derivative. Without loss of generality assume that the derivative is over $\mathbf{x}'$, then it is discontinuous, because infinitely small change of $\mathbf{x}$ will change which half of hyperplane $\mathbf{w}\cdot\mathbf{x}'+b=0$ (corresponding to delta function) is located in the halfspace "allowed" by the Heaviside function: $\mathbf{w}\cdot\mathbf{x}+b>0$. To summarize, $I_2(\mathbf{x},\mathbf{x'})$ is smooth outside of the diagonal $\mathbf{x}=\mathbf{x}'$ and has a first order singularity on it. The first integral $I_1$ is treated similarly, except that one has to differentiate 3 times instead of 1.             

Since the order of singularity is higher in $I_1$, the leading singular term comes from $I_2$. We focus now on deriving its behavior near the diagonal. Consider $\mathbf{x'}=\mathbf{x}+a \hat{\mathbf{n}}$ with small $a>0$ and unit vector $\hat{\mathbf{n}}$. We write $I_2$ as
\begin{equation}\label{eq:I2}
    I_2(\mathbf{x},\mathbf{x'}) = \int\limits_{\mathbf{w}\cdot\mathbf{x}+b>0} p_2(\mathbf{w},b) d\mathbf{w}db \quad-\quad \int\limits_{0<\mathbf{w}\cdot\mathbf{x}+b<-a\mathbf{w}\cdot\hat{\mathbf{n}}} p_2(\mathbf{w},b) d\mathbf{w}db
\end{equation}
The first integral here is the values of $I_2$ on the diagonal, and the second integral is singular, because, e.g., it doesn't change sing with $\hat{\mathbf{n}}\rightarrow-\hat{\mathbf{n}}$ due to being non-negative. We calculate the second integral in \eqref{eq:I2} up to the first order in $|\mathbf{x}-\mathbf{x}'|=a$. The integration is taken in the thin region adjacent the the half of the hyperplane $\mathbf{w}\cdot\mathbf{x}+b=0$ specified by $-\mathbf{w}\cdot\hat{\mathbf{n}}>0$. The thickness of this region at point $\mathbf{w}$ can be calculated using geometric reasoning. The answer is $a(-\mathbf{w}\cdot\hat{\mathbf{n}})_+ / \sqrt{|\mathbf{x}|^2+1}$. This gives us
\begin{equation}\label{eq:I2_mixed}
\begin{split}
    &I_2(\mathbf{x},\mathbf{x'})-I_2(\mathbf{x},\mathbf{x}) = -\frac{a}{\sqrt{|\mathbf{x}|^2+1}}\int\limits_{\mathbf{w}\cdot\mathbf{x}+b=0} (-\mathbf{w}\cdot\hat{\mathbf{n}})_+ p_{2}(\mathbf{w},b) d\mathbf{w}db \; + O(|\mathbf{x}-\mathbf{x'}|^2)
\end{split}
\end{equation}
However, the expression \eqref{eq:I2_mixed} in principle contains both regular and singular parts. We can extract the singular part using the fact that it doesn't change sign under $\hat{\mathbf{n}}\rightarrow-\hat{\mathbf{n}}$, while the regular part does. Since the leading singular part of NTK comes from $I_2$, we can write it as
\begin{equation}\label{eq:Theta_sing}
    \Theta_{\text{sing}}(\mathbf{x},\mathbf{x}') = \frac{\sqrt{|\mathbf{x}|^2+1} }{2}\int\limits_{\mathbf{w}\cdot\mathbf{x}+b=0}  \big|\mathbf{w}\cdot(\mathbf{x}-\mathbf{x'})\big|p_{2}(\mathbf{w},b) d\mathbf{w}db.
\end{equation}
We see that the singularity is of homogeneous type with degree 1, as for the network in NTK regime.  

\newpage
\section{Additional experiments}

\begin{figure}[thb]
    \centering
    \includegraphics[width=0.95\textwidth, clip, trim= 0mm 0mm 0 0mm]{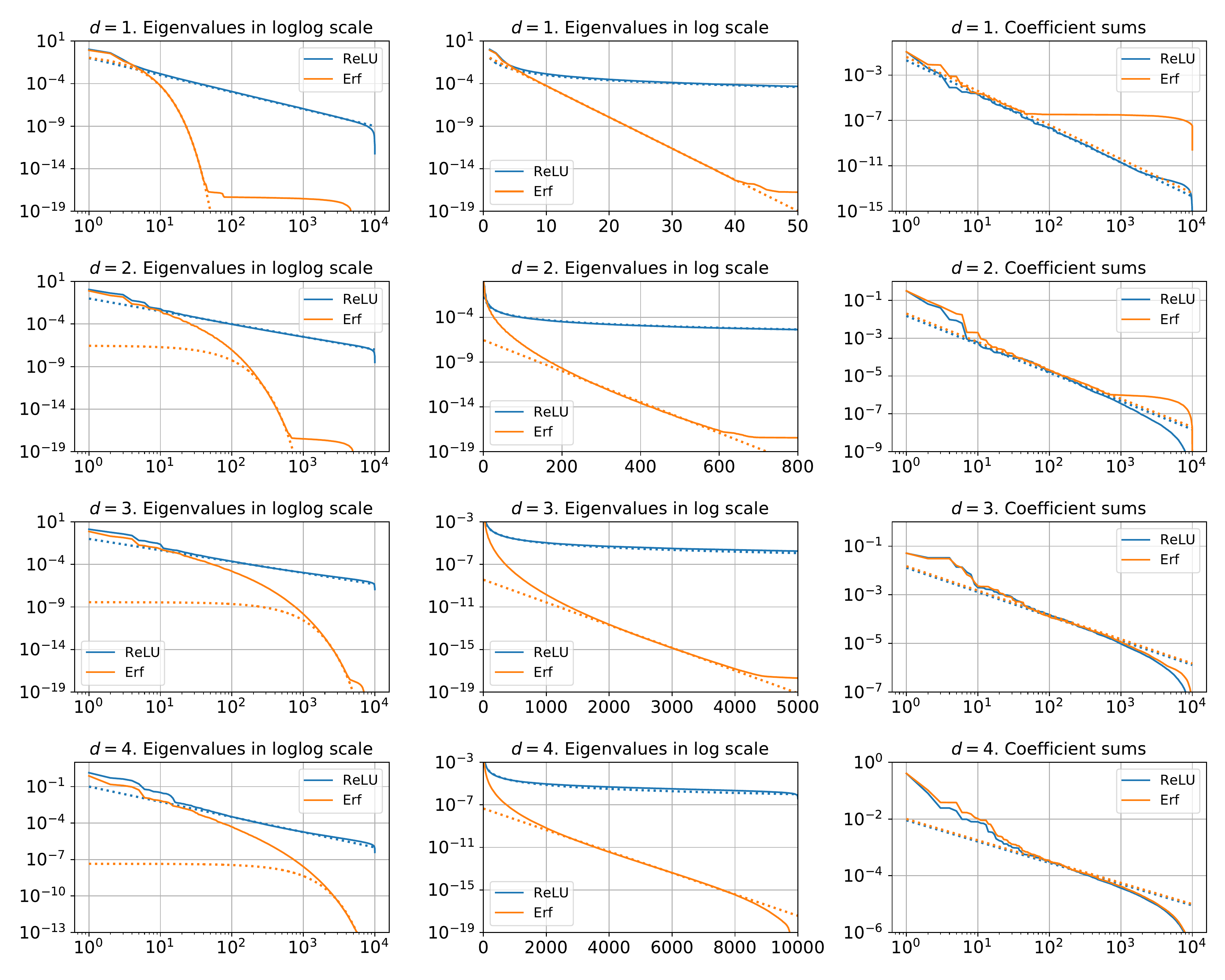}
    \caption{Distribution of eigenvalues $\lambda_n$ and coefficient partial sums $s_n$ for activation functions $\phi(z)=(z)_+$ and $\phi(z)=\operatorname{Erf}(z)$. Target functions in both cases are draws from Gaussian Process modeled by shallow network with NTK parametrization. Dotted lines show analytic expressions fitting the experimental distributions. Eigenvalues for the ReLU activation are fitted with the power law $\lambda_n = \Lambda n^{-1-\tfrac{1}{d}}$, eigenvalues for the Erf activation are fitted with the exponential law $\lambda_n \sim \Lambda e^{-a n}$, and coefficients are fitted with the power law $s_n=Kn^{-\tfrac{3}{d}}$.} 
  \label{fig:activations}
\end{figure}

\begin{figure}[thb]
    \centering
    \includegraphics[width=0.95\textwidth, clip, trim= 0mm 0mm 0 0mm]{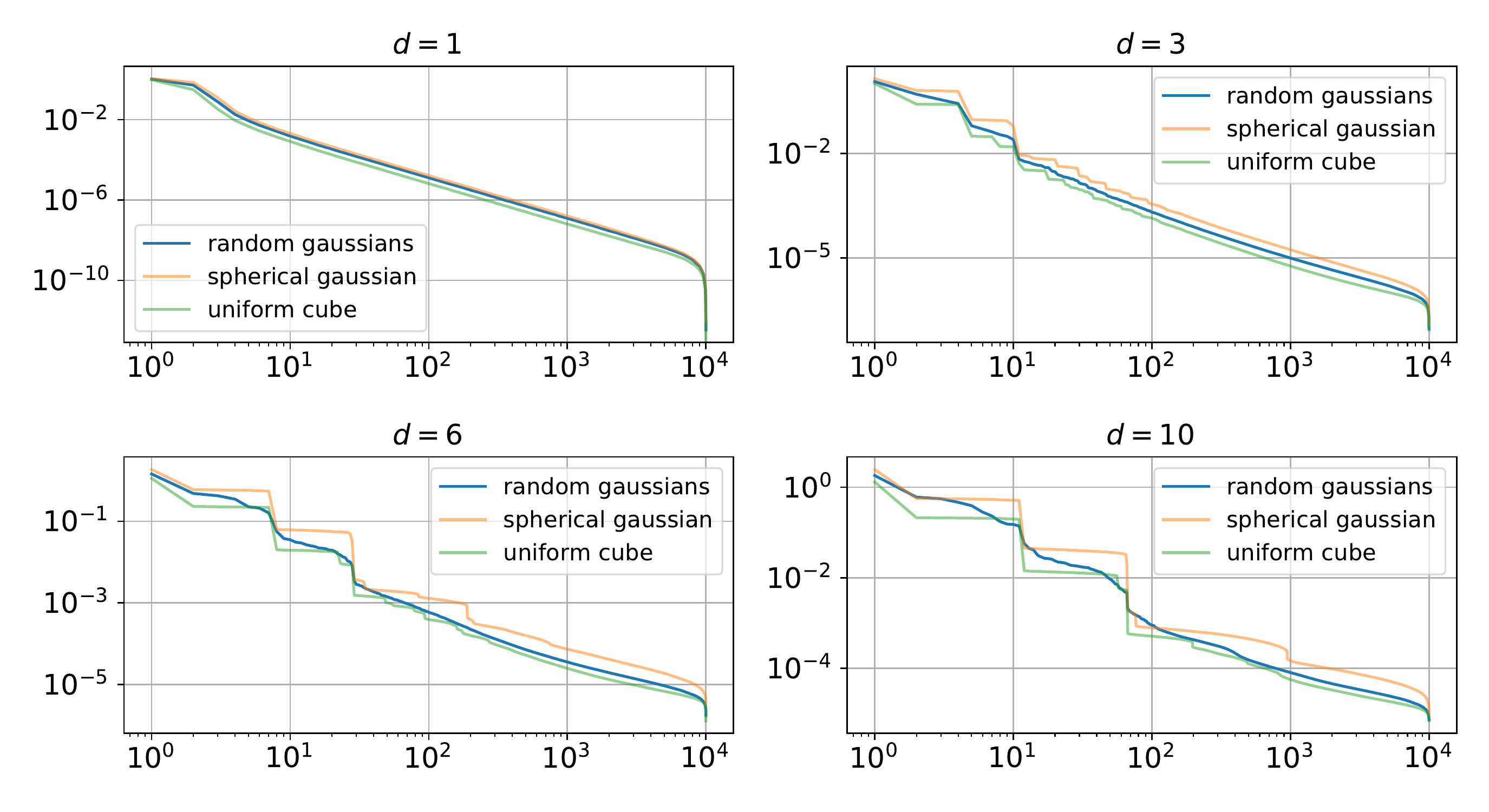}
    \caption{Eigenvalue distribution for different data distribution $\mu(\mathbf{x})$} 
  \label{fig:distributions}
\end{figure}

In the paper we considered only discontinuous activation functions $\phi(z)=(z)_+^{\;q}$. The exponent in power law asymptotic in this case depends on activation smoothness $q$ as $\nu=1+\tfrac{2q-1}{d}$, which means that smoother activations produce NTKs with more quickly decreasing eigenvalues. The natural question is what would be the asymptotic of eigenvalues for smooth activation, although our theory does not apply to such activations. As an example of smooth activation we consider the error function $\phi(z) = \operatorname{Erf}(z)$. The NTK of shallow network with the error function activation is calculated in \cite{NEURIPS2019_0d1a9651}, and it is smooth everywhere. 

In Figure \ref{fig:activations} we can see eigenvalues $\lambda_n$ and coefficient partial sums $s_n$ for NTKs based on the ReLU and Erf activations. We see that the eigenvalues in the Erf case decrease much faster and can be approximately fitted by an exponential law $\lambda_n \sim \Lambda e^{-a n}$. Quite interestingly, for both activations the coefficient distributions $s_n$ behave almost identically, which suggests that the eigenvectors are asymptotically represented by highly oscillating functions regardless of NTK type.         

The second experiment is about data distribution. As we mentioned in Section \ref{sec:exp_det}, the use of symmetric distributions $\mu(\mathbf{x})$ makes the evolution operator $\ta$ also symmetric. In Figure \ref{fig:distributions} we illustrate this point by plotting eigenvalue distribution for normal Gaussian distribution, uniform distribution on $[-1,1]^d$ and average of randomly chosen Gaussians as described in Section \ref{sec:exp_det}. We see that in the case of symmetric data distribution, the eigenvalue distribution has a staircase-like shape, especially for higher dimensions. This is explained by the high degeneracy of  the eigenvalues of symmetric operators. However, for distribution $\mu(\mathbf{x})$ made of randomly chosen Gaussians the staircase-like shape is significantly smoothed, which indicates that such data distribution sufficiently eliminates all the symmetry-based features of corresponding linear operator.  

\end{document}